\documentclass[11pt]{article}
\usepackage[preprint]{acl}
\usepackage{times,latexsym}
\usepackage[T1]{fontenc}
\usepackage[utf8]{inputenc}
\usepackage{microtype,graphicx,booktabs,amsmath,amssymb,amsthm,multirow}
\usepackage{algorithm,algpseudocode}
\allowdisplaybreaks[1]
\usepackage{enumitem,placeins}
\newtheorem{proposition}{Proposition}

\newcommand{\DVC}{\textsc{DVC}}
\newcommand{\E}{\mathbb{E}}
\newcommand{\R}{\mathbb{R}}
\newcommand{\ip}[2]{\langle #1,#2\rangle}
\newcommand{\norm}[1]{\lVert #1\rVert}

\title{Adaptive Data Selection for Text Classification:\\Hierarchical Value with a Local Utility Budget}

\author{
  \textbf{Xiyang Zhang\textsuperscript{1,2}},
  \textbf{Hongzhi Wang\textsuperscript{1}}\thanks{Corresponding author},
  \textbf{Yuanhe Tian\textsuperscript{2,3}},
  \textbf{Yan Song\textsuperscript{4}}
\\
\\
  {\normalfont
  \textsuperscript{1}Harbin Institute of Technology,
  \textsuperscript{2}Zhongguancun Academy}
\\
  {\normalfont
  \textsuperscript{3}Zhongguancun Institute of Artificial Intelligence}
\\
  {\normalfont
  \textsuperscript{4}University of Science and Technology of China}
\\
\\
  {\normalfont\small
    \texttt{s-zhangxiy24@bza.edu.cn} \quad
    \texttt{wangzh@hit.edu.cn}
  }
\\
  {\normalfont\small
    \texttt{tianyuanhe@zgci.ac.cn} \quad
    \texttt{clksong@gmail.com}
  }
}

\begin{document}
\maketitle

\begin{abstract}
Selecting compact training sets from heterogeneous text sources requires balancing target relevance, label reliability, and coverage. For small classifiers, discarding unfamiliar but useful texts can be as harmful as retaining annotation errors. Gradient-based selection estimates immediate target utility, whereas confidence and coverage scores capture complementary properties. Combining them with fixed weights leaves the utility sacrificed for this evidence uncontrolled across candidate pools. We propose Data Value Contribution (\DVC), which gates lexical, hidden-layer, and prediction-level evidence by neighborhood label support and calibrates its influence under a local utility budget; exact MLP gradient contractions and a monotone scoring path make the budget directly verifiable. At 20\% source retention, DVC achieves 75.4 mean macro F1 across eight classification and domain-transfer tasks, exceeding the strongest baseline under shared final training by 0.7 points and random selection by 4.3; an equal-budget tuning study confirms a 0.6-point gain over calibrated quality-only selection.
\end{abstract}

\section{Introduction}
Selecting a compact text training set requires balancing relevance, label reliability, and coverage when the available sources are heterogeneous. A sentiment example from one product domain may provide little supervision for another, even when its label is correct. Conversely, an unusual question can supply useful coverage of a rare class while carrying an incorrect label. Retaining only easy, familiar texts removes some errors but can also discard expressions that the classifier needs to learn. The selection problem is therefore to identify a small set whose labels, content, and coverage serve the target task together.

We study this problem for small text classifiers. Convolutional and shallow classifiers offer practical alternatives when training and deployment resources are limited \citep{kim-2014-convolutional,joulin-etal-2017-bag}. Given several labeled source pools and a labeled target reference set, we select a fixed number of distinct source examples. A freshly initialized classifier then trains on these texts. Selection and final training are separate, so the same subset can be reused with a different classifier. Reference labels guide selection and are excluded from final training in our primary evaluation.

Existing selection methods provide complementary evidence. Corpus matching identifies text similar to the target domain \citep{moore-lewis-2010-intelligent,axelrod-etal-2011-domain}, and learned combinations of similarity and diversity improve transfer \citep{ruder-plank-2017-learning}. Label-quality estimates address unreliable supervision \citep{ren-etal-2020-denoising,lang2022weaksubset,qiao-etal-2022-selfmix}; geometric pruning preserves informative spread \citep{nguyen-he-2025-swift}. Gradient-based methods connect selection to the learning objective more directly \citep{pmlr-v80-ren18a,killamsetty2021glister,pmlr-v139-killamsetty21a}. At a fixed model state, the inner product between an example gradient and the reference gradient predicts the example's first-order effect on reference loss. This utility supplies a useful anchor, but does not capture every reason to retain a text: label reliability, representation coverage, uncertainty, and consistency across training states may favor a different batch.

The missing piece is an explicit rule for how much immediate reference utility to exchange for this additional evidence. A fixed weighted score makes the exchange implicitly. Its coefficient can admit a substantial utility loss in one candidate pool and barely alter selection in another. Simply increasing diversity also risks retaining distinctive annotation errors; filtering solely for label confidence can favor familiar expressions. A selector therefore needs to assess whether additional evidence is supported by the target labels and control how far it can move the batch away from the gradient objective.

We propose Data Value Contribution (\DVC), an adaptive selector that uses reference-gradient utility as an anchor and label-supported evidence as a correction. Layer Value Contribution measures neighborhood label support, gradient agreement, and representation coverage in lexical and hidden views. Global Value Contribution combines reference-gradient impact with uncertainty and prediction stability. Label support gates the auxiliary evidence, so an uncertain or novel example receives less credit when its observed label conflicts with its reference neighborhood. A \emph{local utility budget} calibrates the strength of this correction against the best gradient-only batch under the same class quotas. When the full correction exceeds the allowance, \DVC\ reduces its coefficient and reranks the existing scores. The calibration has a monotone search path and a directly checkable utility bound. Exact layer contractions allow the small MLP to score candidates without storing full per-example parameter gradients.

Across eight classification and natural domain-transfer tasks, \DVC\ achieves 75.4 mean macro F1 at 20\% source retention under common final training, compared with 74.7 for alignment plus quality and 71.1 for random retention. With equal development budgets for the scoring variants, DVC retains a 0.6-point advantage over calibrated quality. Component and representation controls establish the contribution of the hierarchy, and utility-budget analyses verify the cost of its corrections. Text cases illustrate the resulting exchanges; training fresh CNNs on the same subsets demonstrates reuse beyond the selection model. Our contributions are:
\begin{itemize}[leftmargin=*,nosep]
\item We introduce a hierarchical selector that combines lexical, hidden-layer, and global evidence through a neighborhood label-support gate, addressing relevance and annotation quality jointly.
\item We derive a local utility budget and monotone calibration rule that bound the cost of hierarchical reranking, using exact MLP gradient contractions for candidate scoring.
\item We show gains over sixteen selection baselines at 20\% retention on the eight-task average, with equal-budget scoring controls, representation analyses, text cases, and cross-classifier evaluation.
\end{itemize}

\section{Related Work}
\paragraph{Relevance, quality, and coverage.}
Text selection uses corpus similarity, label reliability, and geometric coverage to identify useful examples \citep{ruder-plank-2017-learning,qiao-etal-2022-selfmix,nguyen-he-2025-swift,zeng-etal-2025-data}. These properties can conflict under noisy labels. DVC makes label support a gate on auxiliary evidence and explicitly limits its utility cost.

\paragraph{Target-guided selection.}
Gradient-based reweighting and subset selection connect training examples to a reference objective \citep{pmlr-v80-ren18a,killamsetty2021glister,pmlr-v139-killamsetty21a}; SELCON constrains validation error \citep{pmlr-v139-s21a}. DVC instead bounds the first-order utility exchanged for hierarchical evidence within the current candidate set and class quotas.

\paragraph{Source and layer adaptation.}
Data-mixture reweighting and gradient-guided selection adapt which sources or layers receive training \citep{xie2023doremi,pmlr-v235-xia24c,wang-etal-2025-clusterucb,yao-etal-2026-gast}. DVC uses layer evidence to select reusable texts, rather than to choose parameters for adaptation, and couples source queries to the same reference utility. Appendix~\ref{app:related} gives further comparisons.

\section{Problem and Method}
\label{sec:method}
\begin{figure*}[t]
\centering\includegraphics[width=\textwidth]{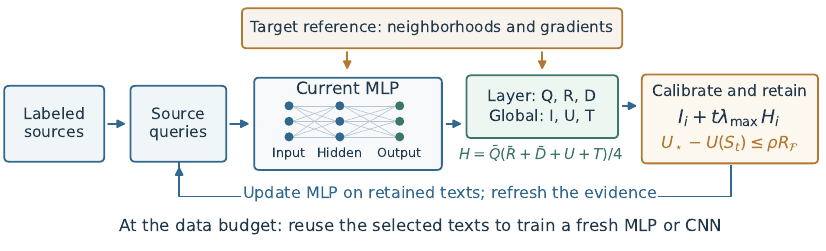}
\caption{DVC combines layer-wise and global evidence while limiting the reference utility lost to its auxiliary correction. Source queries and subset rankings use the same current model; a fresh classifier trains on the retained source examples.}
\label{fig:method}
\end{figure*}

\subsection{Budgeted multi-source classification}
Let $\mathcal S_1,\ldots,\mathcal S_K$ be labeled text sources with a common $C$-class label space. A fixed lexical map $\phi(x)\in\R^d$ feeds an MLP $f_\theta$. A labeled reference set $V$ represents the target task; a separate development set selects hyperparameters. We write
\begin{equation}
 F_V(\theta)=\frac{1}{|V|}\sum_{z\in V}\ell(z;\theta),\qquad
 g_V=\nabla F_V(\theta),
\end{equation}
where $z=(x,y)$ and $\ell$ is cross-entropy. Selection returns exactly $B$ distinct source examples. Reference examples inform selection but are excluded from final training in the primary comparison; their number is reported separately. Test labels enter neither selection nor development.

\DVC\ evaluates small candidate sets at successive model states. A shared random warm subset initializes the selector, followed by optional reference pretraining. Development chooses whether to retain or release warm examples; only final membership counts toward $B$. Each round queries candidate groups, retains a batch, and updates the selector on the accumulated subset. Scores are refreshed before the next round. A fresh classifier trains on the final subset.

\subsection{Global utility and exact layer contractions}
For candidate $i$, let $g_i=\nabla_\theta\ell(z_i;\theta)$. Its reference utility is
\begin{align}
 a_i&=\ip{g_V}{g_i},\qquad I_i=a_i/A,\\
 A&=\max\{\epsilon,\max_{j\in\mathcal C}|a_j|\},
\label{eq:anchor}
\end{align}
where $\mathcal C$ is the current candidate set and $\epsilon=10^{-12}$. Positive $a_i$ predicts a first-order decrease of the reference loss. The common scale $A$ preserves the ordering and sign of this utility.

For an affine layer with input $h_{\ell-1,i}$ and preactivation derivative $\delta_{\ell,i}$,
\begin{equation}
 \nabla_{W_\ell}\ell_i=\delta_{\ell,i}h_{\ell-1,i}^{\top},
 \qquad \nabla_{b_\ell}\ell_i=\delta_{\ell,i}.
\label{eq:outer}
\end{equation}
Writing the reference parameter gradients as $(G_\ell^W,G_\ell^b)$, the exact block inner product is
\begin{equation}
 a_{\ell,i}=\delta_{\ell,i}^{\top}
 (G_\ell^W h_{\ell-1,i}+G_\ell^b),\qquad
 a_i=\sum_\ell a_{\ell,i}.
\label{eq:contraction}
\end{equation}
These identities give exact full-parameter utilities from activations, backpropagated derivatives, and one reference gradient, without storing a separate parameter-gradient vector for each candidate.

\subsection{Hierarchical value signals}
Each auxiliary signal lies in $[0,1]$. Representation signals are averaged over the lexical input and learned hidden representation; gradient agreement is averaged over both parameter blocks. Cosines assign zero to a zero vector; small-denominator protection and clipping are specified in Appendix~\ref{app:algorithm}.

\paragraph{Layer-wise quality.}
A useful representation places similarly labeled texts near one another. For each view $\ell$, let $k'=\min\{k,|V|\}$ and let $N_{\ell,i}$ contain the $k'$ reference examples with highest cosine similarity to candidate $i$. Their weights $\omega_{ij}$ are a softmax of $\tau=10$ times this similarity. With $\pi_c$ denoting the smoothed reference class prior, define local label agreement
\begin{equation}
 u_{\ell,i}=\frac{k'\sum_{j\in N_{\ell,i}}\omega_{ij}\mathbf1\{y_j=y_i\}+\pi_{y_i}}{k'+1}.
\end{equation}
We use $\pi_c=(|V_c|+1/2)/(|V|+C/2)$ and set
\begin{equation}
 Q_{\ell,i}=\sigma\big(\operatorname{logit}(u_{\ell,i})-
 \operatorname{logit}(\pi_{y_i})\big).
\label{eq:quality}
\end{equation}
The score is $1/2$ when neighborhood agreement equals the reference prior. It measures support relative to class frequency; the odds of its population counterpart equal a class-to-complement density ratio (Appendix~\ref{app:quality}).

\paragraph{Layer-wise relevance.}
For the complete parameter block $g_{\ell,i}$, including its bias, define
\begin{equation}
 R_{\ell,i}=\tfrac12\left(1+
 \cos(g_{\ell,i},g_{V,\ell})\right).
\end{equation}
Block normalization gives each layer equal weight. Its numerator and norms follow directly from Equation~\ref{eq:outer}.

\paragraph{Layer-wise diversity.}
A memory $M$ contains up to 256 recently retained examples. Their representations are recomputed at the current state. We use
\begin{equation}
 D_{\ell,i}=\tfrac12\left(1-
 \max_{j\in M}\cos(h_{\ell,i},h_{\ell,j})\right).
\end{equation}
The score is one when $M$ is empty and otherwise measures novelty relative to retained examples.

\paragraph{Global uncertainty and stability.}
For the current prediction $p_i$, uncertainty is its normalized entropy,
\begin{equation}
 U_i=-\frac{\sum_c p_{i,c}\log p_{i,c}}{\log C}.
\end{equation}
Let $p_i^-$ be the prediction of the immediately preceding selector snapshot on the same candidate. Stability is
\begin{equation}
 T_i=\exp\{-\operatorname{JS}(p_i,p_i^-)/s_J\},
\label{eq:stability}
\end{equation}
where $s_J=\max\{10^{-6},\operatorname{median}_{z\in V}\operatorname{JS}(p_z,p_z^-)\}$. Initially $T_i=1$. Both snapshots evaluate the same candidates; the reference scale measures typical prediction changes at that stage.

\paragraph{Combining LVC and GVC.}
With bars denoting the layer averages above, set
\begin{align}
 H_i&=\bar Q_i(\bar R_i+\bar D_i+U_i+T_i)/4,\\
 v_i(t)&=I_i+t\lambda_{\max}H_i,\qquad 0\leq t\leq1.
\label{eq:score}
\end{align}
Quality gates the auxiliary evidence, reducing the contribution of uncertainty and novelty when a label has little neighborhood support. LVC contains the layer-wise relevance and diversity terms; GVC contains impact, uncertainty, and stability (Appendix~\ref{app:algorithm}). Development selects the maximum strength $\lambda_{\max}$; the following rule calibrates its use in each batch.

\subsection{Selection under a utility budget}
\label{sec:allocation}
Class quotas control coverage before ranking. Each of the $b$ batch positions goes to the largest deficit relative to reference label proportions, accounting for retained counts and candidate availability. Let $\mathcal F$ be the resulting family of size-$b$ subsets. Selecting the highest scores within each quota exactly maximizes an additive objective over $\mathcal F$; the gradient-only coverage control uses the same rule.

Write $\mathcal U(S)=b^{-1}\sum_{i\in S}I_i$ and define
\begin{align}
 S(t)&\in\underset{S\in\mathcal F}{\arg\max}
             \sum_{i\in S}v_i(t),\\
 U_\star&=\max_{S\in\mathcal F}\mathcal U(S),\\
 R_{\mathcal F}&=U_\star-\min_{S\in\mathcal F}\mathcal U(S).
\end{align}
We retain the full correction if $U_\star-\mathcal U(S(1))\leq\rho R_{\mathcal F}$, where $\rho\in[0,1]$ is chosen on development data. Otherwise, ten bisection steps on $t\in[0,1]$ return an evaluated feasible value $\hat t$. The gap is nondecreasing in $t$, so the search preserves
\begin{equation}
 U_\star-\mathcal U(S(\hat t))\leq\rho R_{\mathcal F}.
\label{eq:utility-budget}
\end{equation}
Every trial uses the same candidates, model state, and quotas and only reranks the scores. Thus $\rho$ specifies the permitted fraction of the attainable utility range under the current quotas. A zero range admits the full correction. Appendix~\ref{app:utility-budget} proves monotonicity and details the search.

\subsection{Adaptive source queries}
We use EXP3 \citep{auer2002exp3} to allocate queries as source value changes. For $K$ currently available sources,
\begin{equation}
 p_{t,k}=(1-\gamma)\frac{w_{t,k}}{\sum_j w_{t,j}}+\frac{\gamma}{K}.
\end{equation}
The denominator sums over available sources. At query step $t$, a source supplies up to 16 remaining examples. A round collects at most $m$ distinct candidates at a fixed model state. Its group $\mathcal C_t$ returns the mean bounded utility of every queried example:
\begin{equation}
 r_t=\frac1{|\mathcal C_t|}\sum_{i\in\mathcal C_t}
 \frac{1+\tanh(a_i/s_t)}{2},
\label{eq:reward}
\end{equation}
where $s_t=\max\{\epsilon,\operatorname{median}_{z\in V}|\ip{g_V}{g_z}|\}$ is shared across sources. Update the queried source by $w_{t+1,k}=w_{t,k}\exp(\eta r_t/p_{t,k})$, leaving other weights unchanged. We use $\eta=\gamma/K_0$, with initial source count $K_0$ and horizon-based exploration. Changes in source availability restart equal weights. Selected examples and candidates already queried in the round are unavailable; other candidates may reappear later. Appendix~\ref{app:algorithm} gives the complete procedure.

\section{Theoretical Analysis}
\label{sec:analysis}
Fix the current model, candidates, and feasible family $\mathcal F$. Let $a_i^F=\langle\nabla F(\theta),g_i\rangle$ be utility under an objective $F$, and suppose $|a_i^F-a_i|\leq e$ for every candidate. Let $S^*$ maximize its batch mean $\bar a_F$ over $\mathcal F$, and let $\Omega_{\mathcal F}\leq1$ be the range of feasible auxiliary means. Score optimality and the verified budget give
\begin{align}
 \bar a_F(S(\hat t))&\geq\bar a_F(S^*)-\varepsilon_{\rm sel},\\
 \varepsilon_{\rm sel}&=A\min\{\rho R_{\mathcal F},
       \hat t\lambda_{\max}\Omega_{\mathcal F}\}+2e.
\label{eq:calibrated-utility}
\end{align}
For the exact reference objective, $e=0$; under a gradient mismatch of at most $\delta$ and $\|g_i\|\leq G$, take $e=G\delta$. Local smoothness yields a one-step loss-decrease bound with a quadratic remainder. Appendix~\ref{app:proofs} derives both results, the monotone search, and the first accumulated-subset SGD extension. The bound controls the cost of hierarchical reranking at the current model state; final-subset performance is evaluated by fresh training.

\section{Experiments}
\label{sec:experiments}
\subsection{Tasks and protocol}
We evaluate two forms of source heterogeneity. TREC question types \citep{li-roth-2002-learning}, sentence polarity \citep{pang-lee-2005-seeing}, 20 Newsgroups \citep{lang1995newsweeder}, and AG News \citep{zhang2015character} have four label-stratified sources with corruption probabilities $0$, $0.1$, $0.3$, and $0.5$. A corrupted label changes uniformly to another class; clean reference and development sets contain 512 and 256 examples. This construction makes source quality measurable without revealing corruption rates to selection. Test labels are unchanged.

The four-domain Amazon review corpus \citep{blitzer-etal-2007-biographies} provides natural domain transfer. We hold out each domain, select from the other three, and evaluate on the target domain. Its 256 reference and 256 development examples are disjoint from the target test set. These sources differ in vocabulary and content; no synthetic corruption is added. Exact duplicates are removed before splitting \citep{lee-etal-2022-deduplicating}. Appendix~\ref{app:data} gives sizes and preprocessing. BANKING77 \citep{casanueva-etal-2020-efficient} separately tests transfer to a 77-intent task without task-specific tuning.

The main model is a tanh MLP with 64 hidden units and 512-dimensional lexical features. TF--IDF unigrams and bigrams are reduced by truncated SVD \citep{halko2011randomness}; feature fitting excludes test texts. The model has approximately 33,000--38,000 parameters. Every selector produces a subset for a freshly initialized final model. Our primary comparison uses a \emph{shared final-training protocol} within each task: AdamW uses fixed weight decay 0.1; full-pool development selects learning rate, class weighting, and update count. GradMatch's example weights are preserved. Selection hyperparameters are chosen on three separate development seeds, including whether initialization examples must be retained, then fixed for evaluation. All learned selectors receive the same partitions, initial source subset, reference access, and at most 256 candidates per round, retaining up to 32. Ten matched seeds control corruption, source assignment, queries, initialization, and minibatches \citep{reimers-gurevych-2017-reporting,dodge-etal-2019-show}. All model selection uses development data.

We evaluate final-subset quality under shared training, then examine the hierarchy and utility budget separately. Scoring controls receive equal development budgets; component and view controls hold other settings fixed. Fixed-state analyses measure the local exchange, while complete trajectories test its effect on fresh classifiers. Appendix~\ref{app:fair} summarizes the protocols.

\subsection{Comparisons}
We compare random and class-balanced sampling; entropy and EL2N; lexical $k$-center and BADGE-style gradient diversity; raw gradient alignment, gradient cosine, and alignment with class quotas; GradMatch-style nonnegative matching and GLISTER-style Taylor selection. Quality-only and flat six-signal scores test simpler combinations. RHO-LOSS \citep{pmlr-v162-mindermann22a} uses a separately initialized reference-only model to estimate reducible loss. An alignment-plus-quality score, $I_i+\lambda\bar Q_i$, shares DVC's source queries and quotas; its calibrated counterpart applies the same utility budget with $H_i=\bar Q_i$. Together these give sixteen baselines for DVC. Full-pool and reference-only training provide supervision controls. Appendix~\ref{app:baselines} specifies their candidate-limited adaptations. A separate shared-training comparison adds global lexical pruning with SCDP \citep{nguyen-he-2025-swift}.

\subsection{Main results}
\begin{table*}[t]
\centering
\small
\setlength{\tabcolsep}{4pt}
\begin{tabular}{lrrrrrr}
\toprule
Method & TREC & Polarity & 20 News & AG News & Amazon avg. & 8-task avg. \\
\midrule
Random & 68.8 & 62.6 & 49.4 & 86.4 & 75.5 & 71.1 \\
Balanced random & 69.8 & 62.3 & 49.7 & 86.4 & 75.6 & 71.3 \\
Entropy & 70.1 & 62.1 & 47.9 & \textbf{87.2} & 75.3 & 71.1 \\
$k$-center & 70.5 & 61.8 & 48.2 & 86.5 & 76.3 & 71.5 \\
BADGE-style & 67.2 & 62.4 & 48.9 & 86.8 & 75.3 & 70.8 \\
EL2N & 46.7 & 59.2 & 39.1 & 36.5 & 73.7 & 59.5 \\
Alignment & 70.0 & 65.2 & 52.9 & 84.1 & 78.3 & 73.2 \\
Gradient cosine & 69.3 & 65.1 & 54.7 & 86.4 & 77.7 & 73.3 \\
Alignment + coverage & 65.8 & 64.5 & 55.1 & 82.8 & 78.3 & 72.7 \\
GradMatch-style & 71.6 & 63.4 & 55.7 & 75.2 & 77.9 & 72.2 \\
GLISTER-style & 75.5 & 65.4 & 55.8 & 86.9 & 78.2 & 74.5 \\
Quality only & 66.5 & 64.9 & 55.1 & 80.9 & 77.4 & 72.1 \\
Flat six-signal & 70.6 & \textbf{65.9} & \textbf{57.2} & 86.4 & 77.9 & 74.0 \\
RHO-LOSS & 75.6 & 64.5 & 57.0 & 76.6 & 78.3 & 73.4 \\
Alignment + quality & 77.8 & 65.1 & 56.1 & 86.2 & 78.2 & 74.7 \\
Calibrated quality & 77.8 & 65.1 & 56.1 & 86.0 & 78.1 & 74.6 \\
DVC & \textbf{81.1} & \textbf{65.9} & 56.3 & 86.1 & \textbf{78.5} & \textbf{75.4} \\
Full pool & 83.0 & 69.5 & 59.7 & 87.3 & 80.8 & 77.8 \\
Reference only & 67.2 & 63.0 & 43.6 & 66.9 & 75.1 & 67.6 \\
\bottomrule
\end{tabular}
\caption{Comparison at 20\% source retention over ten matched seeds. All methods share the final optimizer, learning rate, class weighting, and update count within each task. Amazon averages four target domains; the last column weights all eight tasks equally. Full-pool and reference-only controls use different training resources and do not enter subset rankings. Bold marks the best subset means, including ties at the displayed precision.}
\label{tab:fair-main}
\end{table*}

At 20\% source retention, \DVC\ achieves the highest eight-task mean among the evaluated selectors: 75.4 macro F1 (Table~\ref{tab:fair-main}). The strongest baseline, alignment plus quality, reaches 74.7. DVC's paired gain is 0.7 points, with a 95\% interval of $[0.3,1.1]$; gains over GLISTER-style selection and flat six-signal scoring are 0.9 ($[0.5,1.3]$) and 1.5 ($[1.0,2.0]$), respectively. All methods share the final optimizer and training hyperparameters, so these differences compare the resulting subsets under the same training conditions.

The largest gain is on TREC: DVC reaches 81.1, compared with 77.8 for alignment plus quality and 75.5 for GLISTER-style selection. It also has the highest four-domain Amazon average, 78.5. Across all eight tasks, DVC improves on random retention by 4.3 points and comes within 2.4 points of full-pool training while retaining only one fifth of the source examples. Reference-only training reaches 67.6, showing that the selected source texts supply useful supervision beyond the reference alone. Per-task results and variation appear in Appendix~\ref{app:fair}; DVC's aggregate advantage also holds when the four Amazon domains receive the weight of one corpus (Appendix~\ref{app:weighting}).

\paragraph{Reusing the selected texts.}
Fresh CNNs train on the MLP-selected subsets with shared task-specific settings and embeddings learned from scratch. Across four tasks and ten paired seeds, DVC reaches 64.5 mean F1, compared with 64.3 for calibrated quality, 63.8 for GLISTER-style selection, and 56.2 for random retention. The gain over random selection is 8.4 points ($[7.2,9.5]$), demonstrating that useful selection transfers beyond the MLP; the comparison with calibrated quality remains close (Appendix~\ref{app:primary-cnn}). Transferring the 20 Newsgroups settings to BANKING77 without task-specific tuning gives DVC 74.1 F1, versus 69.6 for random selection and 74.8 for flat scoring.

\subsection{Hierarchical evidence}
\label{sec:evidence}
\begin{table}[t]
\centering
\small
\setlength{\tabcolsep}{4pt}
\begin{tabular}{lrrr}
\toprule
Views & TREC & Amazon avg. & 8-task avg. \\
\midrule
Two views & \textbf{81.1} & \textbf{78.5} & \textbf{75.4} \\
Input only & 80.7 & 77.9 & 74.9 \\
Hidden only & 73.5 & 78.2 & 74.3 \\
\bottomrule
\end{tabular}
\caption{Complementary representation controls over ten paired seeds. Only auxiliary views change; full-parameter utility and other settings remain fixed. Full task results appear in Appendix~\ref{app:views}.}
\label{tab:view-main}
\end{table}

\begin{table}[t]
\centering
\small
\setlength{\tabcolsep}{4pt}
\begin{tabular}{lrr}
\toprule
Correction & Fixed & Budgeted \\
\midrule
Full hierarchy & 75.2 & 75.1 \\
Quality & 74.8 & 74.5 \\
Compact hierarchy & -- & 75.1 \\
\bottomrule
\end{tabular}
\caption{Independently tuned scoring corrections: eight-task mean macro F1 over ten matched seeds. Each cell receives 36 configurations on three development seeds; final training is shared. Compact scoring removes uncertainty and stability. Complete task results and paired intervals appear in Appendix~\ref{app:independent}.}
\label{tab:upgrade-main}
\end{table}

\paragraph{Independent scoring controls.}
We independently tune full or quality-only corrections with and without calibration, plus a compact hierarchy. Each receives 36 configurations per task on three development seeds, followed by shared final training. This comparison tests whether hierarchical evidence adds value beyond quality alone under equal tuning budgets (Table~\ref{tab:upgrade-main}).

DVC exceeds calibrated quality by 0.6 F1 points ($[0.2,0.9]$), with gains on TREC, 20 Newsgroups, AG News, and Kitchen. Its gain over fixed quality is 0.3 ($[0.0,0.6]$). The full fixed and compact variants reach similar aggregate performance. Thus DVC combines competitive final-subset quality with a directly verified utility constraint; Appendix~\ref{app:independent} separates the contributions of structure and calibration.

\paragraph{Components and views.}
Five-seed component deletions keep the primary comparison's other settings fixed (Appendix~\ref{app:components}). Removing the auxiliary evidence lowers mean F1 by 2.8 points, and removing its quality gate lowers it by 1.4; replacing adaptive source queries with uniform sampling loses 0.5. These contrasts support the gated combination and adaptive queries rather than an unweighted collection of scores. Removing uncertainty and stability together loses 0.5 at fixed settings, while independent tuning recovers similar performance, indicating overlap among the signals.

The view controls change only the representation and corresponding parameter-block auxiliary evidence (Table~\ref{tab:view-main}). Combining lexical and hidden views improves the eight-task average over lexical-only by 0.5 points ($[0.2,0.9]$) and over hidden-only by 1.1 ($[0.8,1.5]$). Lexical evidence is particularly useful on TREC, while hidden-only scoring leads on 20 Newsgroups and AG News. The combined representation draws on both sources of evidence rather than relying on one view throughout.

\paragraph{Label support and class coverage.}
Aggregate purity does not describe which classes retain useful supervision. In the primary TREC comparison, DVC and flat scoring retain 2.0\% and 1.8\% corrupted labels, respectively, yet DVC retains 11.8 clean abbreviation questions on average, compared with 4.2. Quotas reserve opportunities for rare classes, while prior-adjusted neighborhood support distinguishes genuine minority-class texts from examples assigned to those classes by corruption. Removing either quotas or prior adjustment lowers TREC F1 in the structural controls (Appendix~\ref{app:components}). These results connect the hierarchy to the interaction between class imbalance and annotation errors.

\subsection{Utility-budget analysis}
\label{sec:budget-analysis}
\begin{figure*}[t]
\centering
\includegraphics[width=\textwidth]{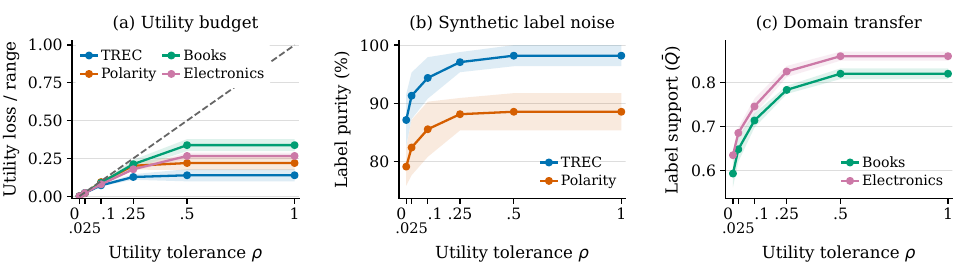}
\caption{Utility loss (a), synthetic-noise label purity (b), and target-reference neighborhood support (c) of newly selected batches. Candidates, parameters, and class quotas are held fixed. Lines average three selection states within each of five seeds; shading gives the standard deviation across seed averages. The dashed line is the permitted utility loss.}
\label{fig:tradeoff}
\end{figure*}

\paragraph{The local trade-off.}
We hold candidates, model parameters, and quotas fixed at early, middle, and late rounds, then vary the correction coefficient or tolerance. Corruption labels are used only to assess the selected batch. On TREC, a coefficient of two raises purity from 87.2\% to 98.2\% over gradient-only selection, at a cost of 13.8\% of the feasible utility range. On Polarity, purity rises from 79.1\% to 88.6\%, at a cost of 21.9\%. A 2.5\% tolerance gives purities of 91.3\% and 82.4\%, with actual utility losses of 2.1\% and 1.9\% (Figure~\ref{fig:tradeoff}). Calibration therefore preserves a useful reliability gain while controlling the exchange against the gradient objective.

On Books and Electronics, the same tolerance raises target-reference neighborhood support from 0.593 to 0.648 and from 0.635 to 0.686, at utility costs of 2.1\% and 2.1\%. These scores measure target-label support rather than independently known annotation accuracy. The common model state makes the utility cost directly comparable across coefficients; complete definitions and additional local comparisons appear in Appendices~\ref{app:tradeoff} and~\ref{app:hierarchy}.

\paragraph{Tolerance and final performance.}
We rerun complete selection trajectories at six prescribed tolerances, holding the other hyperparameters fixed. On TREC, increasing $\rho$ from $0.025$ to $0.5$ raises F1 from 76.3 to 81.6; on Electronics, increasing it from $0.025$ to $0.25$ raises F1 from 79.1 to 81.3 (Appendix~\ref{app:tolerance}). Every run respects its local allowance. Final F1 need not follow a monotone path because changing a selected batch also changes later model states and source queries.

The budget adapts to the task rather than applying the same correction everywhere. In the primary trajectories, it reduces the coefficient in essentially every AG News and Books round, while admitting the full hierarchy on TREC, Polarity, and 20 Newsgroups (Table~\ref{tab:primary-calibration}). In the independently tuned comparison, calibration changes 91.4\% of AG News batches and 91.1\% on Electronics, with no computed budget violation. The key distinction from a fixed coefficient is this explicit control of local utility cost at comparable aggregate F1.

\paragraph{Checking the utility approximation.}
Across 720 fixed-state batches, a full-batch update of $10^{-3}$ gives a pooled correlation of 0.99992 between predicted and measured reference-loss decreases, with matching signs. At $10^{-2}$, the correlation is 0.98919 as curvature becomes more influential. This check supports the first-order approximation used in calibration (Appendix~\ref{app:tradeoff}).

\begin{figure*}[t]
\centering
\includegraphics[width=\textwidth]{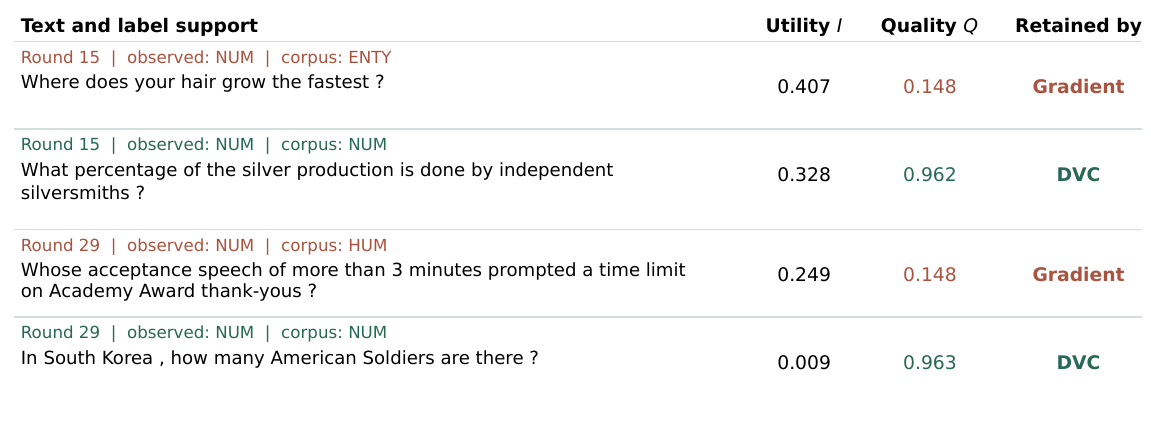}
\caption{TREC examples exchanged under the same observed-label quota at two fixed selector states. The original corpus label is shown only for analysis. Gradient-only selection favors the higher-utility text in each pair; \DVC\ favors the better-supported label. Scores come from development seed 101 and the primary settings.}
\label{fig:case}
\end{figure*}

\subsection{Case study: surface cues and label support}
Figure~\ref{fig:case} shows how label support changes the selection of quantity questions. In round 29, a question asking \emph{whose acceptance speech} contains the phrase \emph{more than 3 minutes}. Its observed label has been corrupted from HUM to NUM, and its utility of 0.249 is higher than that of a genuine \emph{how many} question about soldiers, 0.009. Gradient-only selection retains the former. The reference neighborhood instead gives their NUM labels quality scores of 0.148 and 0.963, so \DVC\ exchanges them while preserving the class quota. Round 15 similarly exchanges a mislabeled question for a genuine percentage question.

At the three saved states, \DVC\ changes 7, 17, and 18 batch members; batch label purity is unchanged at initialization and increases from 93.8\% to 96.9\% midway and from 62.5\% to 95.8\% late in selection. The corresponding utility losses are 0.5\%, 9.6\%, and 21.0\% of the feasible range, within the task's 50\% allowance. These illustrative pairs are chosen by the largest quality difference among same-label clean-for-corrupted exchanges at each state. Appendix~\ref{app:cases} specifies the saved states and corruption diagnosis.

\section{Conclusion}
We introduced \DVC\ for selecting compact text training sets from heterogeneous labeled sources. The method combines layer-wise and global evidence through a label-support gate, then calibrates the correction against the best gradient-only batch under the same class quotas. Exact MLP contractions and monotone calibration provide a verifiable local utility bound without storing full per-example gradients. At 20\% source retention, DVC achieves the highest eight-task mean among the evaluated selectors under shared final training, and equal-budget tuning confirms the benefit of hierarchical evidence over calibrated quality alone. Representation controls and text cases explain how the selected examples preserve useful coverage, while fresh CNN training demonstrates reuse beyond the selection model.

\section*{Limitations}
Our experiments use compact MLPs and CNNs for English text classification. DVC depends on labeled target references; sparse references can leave it below random selection (Appendix~\ref{app:controlled-reference}). Its bound controls local empirical utility, not final-model generalization. Selection incurs additional cost, and gains vary across tasks and retention budgets. Larger architectures, multilingual data, and naturally occurring label noise remain to be evaluated.

\bibliography{references}

\clearpage
\appendix
\setcounter{topnumber}{4}
\setcounter{dbltopnumber}{4}
\setcounter{totalnumber}{6}
\renewcommand{\topfraction}{0.9}
\renewcommand{\dbltopfraction}{0.9}
\renewcommand{\textfraction}{0.05}
\renewcommand{\floatpagefraction}{0.8}
\renewcommand{\dblfloatpagefraction}{0.8}
\raggedbottom

\section{Derivations}
\label{app:proofs}
We derive the MLP contractions, fixed-state selection bounds, and source-feedback estimator. The loss derivatives below are unweighted cross-entropy derivatives; final-training class weights do not enter the selection identities.

\subsection{Parameter-gradient factorization}
Consider an affine block $z_\ell=W_\ell h_{\ell-1}+b_\ell$ in a feed-forward network, with no parameter sharing between blocks. Assume the loss and downstream operations are differentiable at the current example and state. Set $\delta_\ell=\partial\ell/\partial z_\ell$. For coordinate $(r,s)$, the chain rule gives
\begin{equation}
 \frac{\partial\ell}{\partial W_{\ell,rs}}
 =\sum_u\delta_{\ell,u}\mathbf1\{u=r\}h_{\ell-1,s}
 =\delta_{\ell,r}h_{\ell-1,s}.
\end{equation}
Likewise, $\partial\ell/\partial b_{\ell,r}=\delta_{\ell,r}$. The Frobenius inner product of the weight gradient with the reference weight gradient is therefore
\begin{align}
 \ip{G_\ell^W}{\delta_\ell h_{\ell-1}^{\top}}_F
 &=\sum_{r,s}G_{\ell,rs}^W\delta_{\ell,r}h_{\ell-1,s}\\
 &=\delta_\ell^{\top}G_\ell^W h_{\ell-1}.
\end{align}
Adding $\langle G_\ell^b,\delta_\ell\rangle$ proves Equation~\ref{eq:contraction}. Distinct blocks occupy disjoint parameter coordinates, so their inner products sum to the full-parameter inner product, with no cross terms. The block norm used by the relevance score follows from
\begin{align}
 \norm{\delta_\ell h_{\ell-1}^{\top}}_F^2
 &=\norm{\delta_\ell}^2\norm{h_{\ell-1}}^2,\\
 \norm{g_\ell}^2
 &=\norm{\delta_\ell}^2(\norm{h_{\ell-1}}^2+1).
\end{align}
The extra one accounts for the bias.

For the implemented network, $h=\tanh(W_1x+b_1)$, $o=W_2h+b_2$, and $p=\operatorname{softmax}(o)$. The cross-entropy is
\begin{equation}
 \ell(o,y)=\log\sum_{c=1}^{C}\exp(o_c)-o_y.
\end{equation}
Its output derivative is $p-\operatorname{onehot}(y)$. Since the derivative of $\tanh z$ is $1-\tanh^2 z$, backpropagation gives
\begin{align}
 \delta_2&=p-\operatorname{onehot}(y),\\
 \delta_1&=(W_2^{\top}\delta_2)\odot(1-h\odot h).
\end{align}
Substitution supplies the candidate dot products and norms. Individual-loss derivatives contain no candidate-batch normalization; only the reference gradient is averaged, $G_\ell^W=|V|^{-1}\sum_{z\in V}\nabla_{W_\ell}\ell(z)$, and likewise for its bias. A batched forward pass and the two backward expressions suffice.

\subsection{Interpretation of neighborhood quality}
\label{app:quality}
Fix a model state and its representation map. Under a target population, let $H$ denote the resulting representation and fix a class $c$ with true prior $0<\pi_c^{\mathrm{pop}}<1$. Assume the conditional laws of $H$ given $Y=c$ and $Y\ne c$ have densities $f_c$ and $f_{\neg c}$ with respect to a common measure. This measure need not be Lebesgue measure. At points where both densities are positive, write $u_c(h)=P(Y=c\mid H=h)$ and $Z(h)=\pi_c^{\mathrm{pop}}f_c(h)+(1-\pi_c^{\mathrm{pop}})f_{\neg c}(h)$. Bayes' rule gives
\begin{align}
 u_c(h)&=\frac{\pi_c^{\mathrm{pop}}f_c(h)}{Z(h)},\\
 \operatorname{logit}u_c(h)-\operatorname{logit}\pi_c^{\mathrm{pop}}
 &=\log\frac{f_c(h)}{f_{\neg c}(h)}.
\end{align}
The common $Z(h)$ cancels in the odds; subtracting the prior log-odds then removes $\pi_c^{\mathrm{pop}}/(1-\pi_c^{\mathrm{pop}})$. Applying $\sigma(z)=\exp(z)/(1+\exp(z))$ yields
\begin{equation}
 Q_c^{\mathrm{pop}}(h)=\frac{f_c(h)}{f_c(h)+f_{\neg c}(h)}.
\end{equation}
Hence $Q_c^{\mathrm{pop}}>1/2$ exactly when $f_c>f_{\neg c}$. If exactly one density vanishes, the result extends by its corresponding zero or one limit; points where both vanish carry no population probability. This is an interpretation of the ideal population score. Finite reference neighborhoods, the prior pseudocount, and clipping do not turn the implemented score into an exact density-ratio estimate. In particular, fixed neighborhood size supplies no consistency claim. The interpretation applies separately to each representation view, rather than to their average.

\subsection{Boundedness of the signals}
For any two nonzero vectors, Cauchy--Schwarz gives a cosine in $[-1,1]$. The stipulated zero-vector value is also in this interval. Hence $R$ and $D$ lie in $[0,1]$. Entropy satisfies $0\leq\mathcal H(p)\leq\log C$ for $C\geq2$, with the upper bound following from $\operatorname{KL}(p\|\mathbf1/C)=\log C-\mathcal H(p)\geq0$. Thus normalized uncertainty lies in $[0,1]$. Nonnegativity of Jensen--Shannon divergence gives stability in $(0,1]$ in real arithmetic. Quality's smoothing and sigmoid give a value in $(0,1)$ (Appendix~\ref{app:algorithm}). The numerical implementation clips relevance, diversity, and uncertainty to $[0,1]$ and keeps stability nonnegative. Averages and products preserve $[0,1]$; hence $H_i\in[0,1]$, as required by Proposition~\ref{prop:utility}. These bounds concern the scoring signals; neighborhood quality is relative label support, not a calibrated probability of annotation correctness.

\subsection{Utility retained by the correction}
\begin{proposition}[Utility retained by hierarchical correction]
\label{prop:utility}
Fix a candidate set, a positive scale $A$, and a nonempty family $\mathcal F$ of feasible size-$b$ subsets, with $b\geq1$. Let $\Omega_{\mathcal F}=\max_{S\in\mathcal F}\bar H(S)-\min_{S\in\mathcal F}\bar H(S)$, where $\bar H(S)=b^{-1}\sum_{i\in S}H_i$. Suppose $\lambda\geq0$, $|\hat a_i-a_i|\leq e$, and $0\leq H_i\leq1$. If $\hat S$ maximizes $\sum_{i\in S}(\hat a_i/A+\lambda H_i)$ over $\mathcal F$, and $S^*$ maximizes $\sum_{i\in S}a_i$ over the same family, then
\begin{equation}
 \frac1b\sum_{i\in\hat S}a_i
 \geq \frac1b\sum_{i\in S^*}a_i-A\lambda\Omega_{\mathcal F}-2e.
\label{eq:utility-bound}
\end{equation}
\end{proposition}

\paragraph{Proof.}
Write $\bar a(S)=b^{-1}\sum_{i\in S}a_i$, with analogous notation for $\hat a$ and $H$. Optimality of $\hat S$ for the computed score over the same feasible family yields
\begin{equation}
 \overline{\hat a}(\hat S)/A+\lambda\bar H(\hat S)
 \geq \overline{\hat a}(S^*)/A+\lambda\bar H(S^*).
\end{equation}
Multiplying by the positive $A$ and rearranging,
\begin{equation}
 \overline{\hat a}(S^*)-\overline{\hat a}(\hat S)
 \leq A\lambda\{\bar H(\hat S)-\bar H(S^*)\}
 \leq A\lambda\Omega_{\mathcal F}.
\end{equation}
The pointwise error hypothesis implies
$|\overline{\hat a}(S)-\bar a(S)|\leq e$ for every feasible $S$: the absolute value of an average is at most the average of the absolute values. Consequently,
\begin{align}
 \bar a(S^*)-\bar a(\hat S)
 &=\bar a(S^*)-\overline{\hat a}(S^*)\nonumber\\
 &\quad+\overline{\hat a}(S^*)-\overline{\hat a}(\hat S)\nonumber\\
 &\quad+\overline{\hat a}(\hat S)-\bar a(\hat S)\\
 &\leq e+A\lambda\Omega_{\mathcal F}+e.
\end{align}
This proves Equation~\ref{eq:utility-bound}; $\lambda=e=0$ recovers exact feasible utility maximization. The comparator uses the same candidates and quotas: any change in source queries or coverage constraints changes $\mathcal F$ and requires a separate comparison.

\paragraph{Why the quota optimizer is exact.}
Let $\mathcal C_c$ be the candidates with observed label $c$. Fix integer quotas $q_c$ with $\sum_cq_c=b$ and $0\leq q_c\leq|\mathcal C_c|$, and define
\begin{equation}
 \mathcal F=\{S\subseteq\mathcal C:
             |S\cap\mathcal C_c|=q_c\text{ for every }c\}.
\end{equation}
This family is nonempty: each class supplies its own $q_c$ members. If a feasible subset omits a higher-scoring member of a class while including a lower-scoring member of that class, exchanging the two preserves feasibility and increases its objective. Repeating exchanges gives per-class top-$q_c$ selection; equal-score exchanges leave the objective unchanged. This proves exact additive maximization. The deficit-based allocation fixes these quotas before ranking and cannot exhaust all class capacities before $b$ positions have been allocated, since $b\leq|\mathcal C|$. Thus the same family applies to every calibration trial, even though it can change at the next round.

\paragraph{Computing the correction range.}
The quantity $\Omega_{\mathcal F}$ requires no enumeration of subsets under these quotas. For class $c$, order its $n_c^{\mathcal C}=|\mathcal C_c|$ auxiliary values as $H_{c,(1)}\leq\cdots\leq H_{c,(n_c^{\mathcal C})}$. The same exchange argument, applied separately to maximizing and minimizing the auxiliary sum, gives
\begin{equation}
 \Omega_{\mathcal F}=\frac1b\sum_c\sum_{j=1}^{q_c}
 \left(H_{c,(n_c^{\mathcal C}-j+1)}-H_{c,(j)}\right).
\label{eq:quota-oscillation}
\end{equation}
Empty sums are zero. Without quotas, subtract the sum of the $b$ smallest candidate values from the $b$ largest. Both formulas remain exact when the maximizing and minimizing sets overlap.

\subsection{Local loss change}
\begin{proposition}[One-step loss decrease]
\label{prop:descent}
Under the hypotheses of Proposition~\ref{prop:utility}, define $a_i=\ip{\nabla F(\theta)}{g_i}$ and $\bar g_S=b^{-1}\sum_{i\in S}g_i$. Let $\varepsilon_{\mathrm{sel}}$ be a verified utility-loss bound: use $A\lambda\Omega_{\mathcal F}+2e$, or the calibrated bound in Equation~\ref{eq:utility-budget-bound}. For $\alpha\geq0$, if $F$ has an $L_F$-Lipschitz gradient on the segment from $\theta$ to $\theta-\alpha\bar g_{\hat S}$, then
\begin{align}
 F(\theta)-F(\theta-\alpha\bar g_{\hat S})
 &\geq\alpha\{\bar a(S^*)-\varepsilon_{\mathrm{sel}}\}\nonumber\\
 &\quad-\tfrac12 L_F\alpha^2\norm{\bar g_{\hat S}}^2.
\label{eq:local-update}
\end{align}
\end{proposition}

\paragraph{Proof.}
Let $d=\bar g_{\hat S}$ and $q(t)=\nabla F(\theta-td)-\nabla F(\theta)$. The case $\alpha=0$ is immediate. For $\alpha>0$, the fundamental theorem of calculus gives
\begin{align}
 &F(\theta-\alpha d)-F(\theta)\nonumber\\
 &=-\int_0^\alpha\ip{\nabla F(\theta-td)}{d}\,dt\\
 &=-\alpha\ip{\nabla F(\theta)}{d}\nonumber\\
 &\quad-\int_0^\alpha
 \ip{q(t)}{d}\,dt.
\end{align}
The gradient-Lipschitz condition and Cauchy--Schwarz give
\begin{align}
 \left|\ip{\nabla F(\theta-td)-\nabla F(\theta)}{d}\right|
 &\leq L_F t\norm{d}^2,\\
 \int_0^\alpha L_Ft\norm{d}^2dt
 &=\tfrac12 L_F\alpha^2\norm{d}^2.
\end{align}
Therefore,
\begin{equation}
 F(\theta)-F(\theta-\alpha d)
 \geq\alpha\ip{\nabla F(\theta)}{d}
       -\tfrac12L_F\alpha^2\norm{d}^2.
\end{equation}
Linearity of the inner product yields
$\ip{\nabla F(\theta)}{d}=b^{-1}\sum_{i\in\hat S}a_i$.
Applying the verified utility bound proves the result. In particular, let $\kappa=\bar a(S^*)-\varepsilon_{\mathrm{sel}}$. If $\kappa>0$ and $L_F\norm{d}^2>0$, then any $0<\alpha<2\kappa/(L_F\norm{d}^2)$ makes the lower bound strictly positive. If $L_F=0$, it suffices that $\alpha>0$ and $\kappa>0$. Thus a positive utility margin and a suitable step size suffice for descent; utility preservation by itself does not assert a positive margin. This describes an unweighted gradient step; momentum and AdamW are not covered by this update bound.

For the finite-reference tanh MLP loss with finite input vectors, log-sum-exp and tanh have continuous derivatives of every order. The Hessian is therefore continuous. On a compact update segment its operator norm has a finite maximum $L_F$. For any two points $\theta_1,\theta_2$ on that segment, set $\Delta=\theta_2-\theta_1$ and $\theta_t=\theta_1+t\Delta$. Then
\begin{align}
 \nabla F(\theta_2)-\nabla F(\theta_1)
 &=\int_0^1\nabla^2 F(\theta_t)\Delta\,dt,\\
 \|\nabla F(\theta_2)-\nabla F(\theta_1)\|
 &\leq L_F\|\Delta\|.
\end{align}
This establishes existence of the local smoothness constant, without identifying its numerical value. Population-objective smoothness remains an assumption.

\subsection{Reference mismatch}
\label{app:stale}
Let $u=\nabla F(\theta)$ be the desired target gradient and $\hat u=g_V$. If $\norm{u-\hat u}\leq\delta$ and $\norm{g_i}\leq G$, then
\begin{align}
 |\ip{u}{g_i}-\ip{\hat u}{g_i}|
 &=|\ip{u-\hat u}{g_i}|\\
 &\leq\norm{u-\hat u}\norm{g_i}\leq\delta G.
\end{align}
This deterministic bound needs no independence after repeated reference use.

\subsection{Source feedback}
\label{app:bandit}
EXP3 provides exploratory source queries \citep{auer2002exp3}. Before a reached query $t$, let $\mathcal H_{t-1}$ contain the current model, reference gradient, pool masks, source weights, remaining query allowance, and positive reward scale $s_t$. Source probabilities are measurable with respect to this history. If available source $I_t$ supplies a nonempty random group $\mathcal C_t$, its actual feedback is
\begin{equation}
 R_t=\frac1{|\mathcal C_t|}\sum_{i\in\mathcal C_t}
       \frac{1+\tanh(a_i/s_t)}2\in[0,1].
\end{equation}
Every queried example contributes, including those not retained. Let $\mathcal A_{t,k}$ be source $k$'s remaining candidates before the query and $q_{t,k}\geq1$ its group size, clipped by source availability and the remaining allowance. Conditional on the history and $I_t=k$, this size is fixed. Uniform sampling without replacement gives each $i\in\mathcal A_{t,k}$ inclusion probability $q_{t,k}/|\mathcal A_{t,k}|$. With $z_i=(1+\tanh(a_i/s_t))/2$, it follows that
\begin{align}
 r_{t,k}&=\E[R_t\mid\mathcal H_{t-1},I_t=k]\\
 &=\frac1{q_{t,k}}\sum_{i\in\mathcal A_{t,k}}
       \frac{q_{t,k}}{|\mathcal A_{t,k}|}z_i\\
 &=\frac1{|\mathcal A_{t,k}|}\sum_{i\in\mathcal A_{t,k}}z_i.
\end{align}
Thus short groups do not change the conditional source mean. Write $\E_t=\E[\,\cdot\mid\mathcal H_{t-1}]$ and let $K_t$ be the number of available sources. For an available source, $p_{t,k}\geq\gamma/K_t>0$. The importance-weighted feedback $\hat r_{t,k}=R_t\mathbf1\{I_t=k\}/p_{t,k}$ satisfies
\begin{align}
 \E_t\hat r_{t,k}
 &=\sum_{j:p_{t,j}>0}p_{t,j}\mathbf1\{j=k\}\frac{r_{t,j}}{p_{t,k}}\\
 &=p_{t,k}\frac{r_{t,k}}{p_{t,k}}=r_{t,k}.
\end{align}
Thus feedback is unbiased for the current conditional source mean. No estimator is assigned to an unavailable source with probability zero. The model and reward scale remain fixed within a round; availability and probabilities may change after each query. Across rounds, training also changes the source means. Conditional unbiasedness at each reached query does not assert that a randomly stopped average estimates a stationary source value. Equal weights restart when availability changes, and we do not invoke a fixed-horizon regret bound for the complete restarted procedure.

\subsection{A budget on reference utility}
\label{app:utility-budget}
We derive the scoring path's monotonicity, its calibrated utility bound, and the first accumulated-set SGD update.

\paragraph{The scoring path.}
Fix a finite candidate set and a nonempty family $\mathcal F$ of its size-$b$ subsets, with $b\geq1$. Set $I_i=\hat a_i/A$, where $A>0$, and let $v_i=I_i+\lambda_{\max}H_i$ be the full score, with $\lambda_{\max}\geq0$ and $0\leq H_i\leq1$. Write $d_i=v_i-I_i=\lambda_{\max}H_i$. For $0\leq t\leq1$, define
\begin{align}
 \mathcal U(S)&=\frac1b\sum_{i\in S}I_i,\\
 \bar d(S)&=\frac1b\sum_{i\in S}d_i,\\
 S(t)&\in\underset{S\in\mathcal F}{\arg\max}\,
       \{\mathcal U(S)+t\bar d(S)\}.
\label{eq:utility-path}
\end{align}
Keep $\mathcal F$, including any class quotas, fixed. Define
\begin{align}
 U_\star&=\max_{S\in\mathcal F}\mathcal U(S),\\
 U_-&=\min_{S\in\mathcal F}\mathcal U(S),\\
 R_{\mathcal F}&=U_\star-U_-,\\
 G(t)&=U_\star-\mathcal U(S(t)).
\label{eq:utility-budget-range}
\end{align}
The permitted loss is $\rho R_{\mathcal F}$, $\rho\in[0,1]$. Applying the same optimizer to $I$ and $-I$ gives both range endpoints.

\paragraph{Monotonicity.}
Take $0\leq t_1<t_2\leq1$, and abbreviate $U_j=\mathcal U(S(t_j))$ and $D_j=\bar d(S(t_j))$. Optimality at each value of $t$ gives
\begin{align}
 U_1+t_1D_1&\geq U_2+t_1D_2,\\
 U_2+t_2D_2&\geq U_1+t_2D_1.
\end{align}
Adding these inequalities yields
\begin{equation}
 (t_2-t_1)(D_2-D_1)\geq0,
\end{equation}
so $D_2\geq D_1$. Substitution into the first inequality gives
\begin{equation}
 U_1-U_2\geq t_1(D_2-D_1)\geq0.
\end{equation}
Thus utility is nonincreasing and $G(t)$ nondecreasing. The argument permits any maximizer at a tie, and needs neither continuity of the subset path nor nonnegative $d_i$. It requires the same feasible family, anchor, and correction at every $t$.

\paragraph{Auxiliary improvement and the supported frontier.}
Let $S_0\in\arg\max_{S\in\mathcal F}\mathcal U(S)$ and $h=t\lambda_{\max}$. Score optimality gives
\begin{equation}
 \mathcal U(S(t))+h\bar H(S(t))
 \geq\mathcal U(S_0)+h\bar H(S_0).
\end{equation}
Rearranging and using $G(t)\geq0$, at the calibrated value $\hat t$ with $h=\hat t\lambda_{\max}$ and $\hat S=S(\hat t)$,
\begin{equation}
 0\leq G(\hat t)\leq h\{\bar H(\hat S)-\bar H(S_0)\}.
\label{eq:exchange}
\end{equation}
For $h>0$, dividing by $h$ shows that auxiliary evidence cannot decrease. Since both coefficients in $\mathcal U+h\bar H$ are then positive, a feasible set that weakly improves both means and strictly improves one would have a larger score, contradicting optimality. Thus $S(\hat t)$ is Pareto efficient among feasible subsets in these two computed means. At $h=0$, utility is optimal but evidence need not be. This is a statement about the scoring signals, rather than independently measured label reliability.

\paragraph{A feasible search.}
If $G(1)\leq\rho R_{\mathcal F}$, retain the full score. Otherwise start with endpoints $l=0$ and $u=1$. The lower endpoint is feasible because $G(0)=0$, and the upper endpoint is infeasible. At the midpoint $t=(l+u)/2$, compute $S(t)$ and its actual gap. Replace $l$ by $t$ only if $G(t)\leq\rho R_{\mathcal F}$; otherwise replace $u$ by $t$. After each replacement, the stored lower-endpoint subset remains feasible. Ten iterations therefore return $\hat t=l$ and $\hat S=S(\hat t)$ satisfying
\begin{align}
 G(\hat t)&\leq\rho R_{\mathcal F},\\
 0\leq t_{\sup}-\hat t&\leq2^{-10},
\label{eq:utility-budget-search}
\end{align}
where $t_{\sup}=\sup\{t:G(t)\leq\rho R_{\mathcal F}\}$. Monotonicity keeps this supremum in the bracket, whose width halves each iteration; the supremum need not itself be feasible. The $2^{-10}$ bound concerns the coefficient resolution in real arithmetic, rather than an error in subset utility or an unrestricted constrained optimum. The procedure returns the stored, evaluated subset rather than selecting it again. The full score is feasible for $R_{\mathcal F}=0$ or $\rho=1$. Discrete jumps can leave budget unused, and the search is restricted to subsets supported by this score path.

\paragraph{Utility with an imperfect reference gradient.}
Let $S^*$ maximize the true mean utility $\bar a(S)$ over $\mathcal F$, and assume $|\hat a_i-a_i|\leq e$ for every candidate. Then
\begin{align}
 \bar a(S^*)
 &\leq\overline{\hat a}(S^*)+e\\
 &\leq A U_\star+e\\
 &\leq A\mathcal U(\hat S)+A\rho R_{\mathcal F}+e\\
 &=\overline{\hat a}(\hat S)+A\rho R_{\mathcal F}+e\\
 &\leq\bar a(\hat S)+A\rho R_{\mathcal F}+2e.
\label{eq:utility-budget-error}
\end{align}
The first and last steps use the uniform error bound, the second uses reference-utility optimality, and the third uses the verified gap. Since $\hat S$ also maximizes $I+\hat t\lambda_{\max}H$, Proposition~\ref{prop:utility} applies with $\lambda=\hat t\lambda_{\max}$. Combining the two bounds gives
\begin{align}
 \varepsilon_{\mathrm{sel}}
 &=A\min\{\rho R_{\mathcal F},
          \hat t\lambda_{\max}\Omega_{\mathcal F}\}+2e,\\
 \bar a(\hat S)&\geq\bar a(S^*)-\varepsilon_{\mathrm{sel}}.
\label{eq:utility-budget-bound}
\end{align}
Exact current gradients with $F=F_V$ give $e=0$. Multiplication by $A$ converts normalized utility to raw inner-product units. The evaluated gap gives the sharper, batch-specific certificate
\begin{align}
 \bar a(S^*)-\bar a(\hat S)
 &\leq A G(\hat t)+2e\\
 &\leq A\min\{\rho R_{\mathcal F},
       \hat t\lambda_{\max}\Omega_{\mathcal F}\}+2e.
\label{eq:observed-utility-bound}
\end{align}
For the first inequality, replace $\rho R_{\mathcal F}$ in Equation~\ref{eq:utility-budget-error} by the identity $U_\star-\mathcal U(\hat S)=G(\hat t)$. For the second, combine feasibility with Equation~\ref{eq:exchange} and $\bar H(\hat S)-\bar H(S_0)\leq\Omega_{\mathcal F}$. When $e=0$ in real arithmetic, the first inequality is an equality: the certificate is the exact reference-utility difference from the best batch with the same quotas and candidates. It remains meaningful when $R_{\mathcal F}=0$, without dividing by the range. Computational gap checks use the stored anchor, as described in Appendix~\ref{app:algorithm}; gradient approximation or roundoff is not a bound on population-reference mismatch.

\paragraph{The first accumulated-set SGD step.}
\label{app:accumulated}
Condition on the current model, candidates, and completed selection. Let $O$ contain the $n$ previously retained examples. The candidate set is disjoint from $O$, so every $S\in\mathcal F$, including $S^*$ and $\hat S$, contains $b$ new examples. Put $T=O\cup\hat S$. At this same parameter state define
\begin{align}
 \mu&=\frac1{n+b}\sum_{i\in T}g_i,\\
 \sigma^2&=\frac1{n+b}\sum_{i\in T}\norm{g_i-\mu}^2,\\
 u^*&=\frac1{n+b}\sum_{i\in O\cup S^*}a_i.
\end{align}
Here $a_i=\langle\nabla F(\theta),g_i\rangle$. The common old-set contribution and Equation~\ref{eq:utility-budget-bound} imply
\begin{align}
 \ip{\nabla F(\theta)}{\mu}
 &=\frac{\sum_{i\in O}a_i+b\bar a(\hat S)}{n+b}\\
 &\geq u^*-\frac{b}{n+b}\varepsilon_{\mathrm{sel}}.
\end{align}
For $s\geq1$, draw $J_1,\ldots,J_s$ independently and uniformly from $T$, with replacement, and set $\hat g=s^{-1}\sum_jg_{J_j}$. For $Z_j=g_{J_j}-\mu$, independence and centering give $\E Z_j=0$, $\E\norm{Z_j}^2=\sigma^2$, and $\E\ip{Z_j}{Z_k}=0$ when $j\ne k$. Consequently,
\begin{align}
 \E\hat g&=\mu,\\
 \E\norm{\hat g}^2
 &=\norm{\mu}^2+
   \frac1{s^2}\sum_{j,k}\E\ip{Z_j}{Z_k}\\
 &=\norm{\mu}^2+\frac{\sigma^2}{s}.
\end{align}
For $\alpha\geq0$, if a common $L_F$ bounds gradient variation along every possible update segment, the smoothness inequality yields
\begin{align}
 &\E[F(\theta)-F(\theta-\alpha\hat g)]\nonumber\\
 &\quad\geq\alpha\ip{\nabla F(\theta)}{\mu}
       -\frac{L_F\alpha^2}{2}\E\norm{\hat g}^2\\
 &\quad\geq\alpha\left(u^*-
          \frac{b}{n+b}\varepsilon_{\mathrm{sel}}\right)\nonumber\\
 &\qquad\quad-\frac{L_F\alpha^2}{2}
       \left(\norm{\mu}^2+\frac{\sigma^2}{s}\right).
\label{eq:utility-budget-sgd}
\end{align}
This expectation is conditional on selection and concerns the first unweighted, post-selection SGD minibatch, which is sampled with replacement in our selector. It does not average over different selected subsets. For a finite training set and fixed minibatch size there are finitely many possible update segments; continuity of the finite-reference MLP Hessian therefore supplies a common finite $L_F$ on their union. Its value is not estimated here. Utilities must be refreshed to establish a comparison at a later model state; the present bound covers only the first step. Final subset quality is assessed by fresh training.

\section{Complete Selection Procedure}
\label{app:algorithm}
\subsection{Representations, probabilities, and boundary cases}
The selector $f_\theta(x)=W_2\tanh(W_1\phi(x)+b_1)+b_2$ uses views $h_0=\phi(x)$ and $h_1=\tanh(W_1\phi(x)+b_1)$. Quality and diversity average these views; relevance averages the complete affine blocks. Input features stay fixed; model-dependent quantities are refreshed before each selection round. For $\lambda=t\lambda_{\max}$, the layer/global notation is $\mathrm{LVC}_i=\bar Q_i(\bar R_i+\bar D_i)/2$, $\mathrm{GVC}_i=I_i+\lambda\bar Q_i(U_i+T_i)/4$, and $v_i=\lambda\mathrm{LVC}_i/2+\mathrm{GVC}_i$.

For quality, cosine similarities identify $k'=\min\{k,|V|\}$ reference neighbors. With $s_{ij}$ denoting similarity in a given view, the weights are
\begin{equation}
 \omega_{ij}=\frac{\exp(\tau s_{ij})}
 {\sum_{j'\in N_i}\exp(\tau s_{ij'})},\qquad \tau=10.
\end{equation}
The multiplier is fixed across tasks; neighborhood size is selected on development data. Weighted label agreement $a\in[0,1]$ gives $u=(k'a+\pi_{y_i})/(k'+1)\in(0,1)$. Since $a=\pi_{y_i}$ implies $u=\pi_{y_i}$, the prior-centered score is neutral at prior-level agreement. The implementation clips $u$ to $[10^{-6},1-10^{-6}]$ before its logit. This clipping is inactive at prior-level agreement for all references used here; candidate labels remain unchanged.

For stability, Jensen--Shannon divergence is
\begin{align}
 J(p,q)&=\tfrac12\sum_c p_c\log(p_c/m_c)\nonumber\\
       &\quad+\tfrac12\sum_c q_c\log(q_c/m_c),\\
 m_c&=(p_c+q_c)/2.
\end{align}
Using $0\log0=0$, set $s_J=\max\{10^{-6},\operatorname{median}_{z\in V}J(p_z,p_z^-)\}$ and $T=\exp(-J/s_J)\in(0,1]$. Initially the snapshots coincide; subsequently the previous snapshot precedes the latest selector update. Both evaluate the same candidates and references. Negative roundoff in $J$ is clipped to zero.

Let $\epsilon=10^{-12}$. Representation similarity uses
\begin{equation}
 \frac{\langle h,h'\rangle}
 {\max\{\|h\|,\epsilon\}\max\{\|h'\|,\epsilon\}},
\end{equation}
whereas block-gradient similarity divides its contraction by $\max\{\|g_{\ell,i}\|\|g_{V,\ell}\|,\epsilon\}$. Both assign zero to a zero vector and agree with cosine when the floors are inactive. Block-gradient cosines are clipped to $[-1,1]$; diversity and uncertainty are clipped to $[0,1]$. The utility scale is floored at $\epsilon$. If all utility contractions vanish, the impact term and its feasible range are zero, so the auxiliary score determines the ranking. Diversity is one when the retained set is empty. Otherwise its memory contains the most recent 256 retained examples, whose representations are evaluated at the current state.

\subsection{Calibrating the auxiliary correction}
Let $\mathcal T_{\mathcal F}(z)$ maximize the additive score by per-class top-quota or ordinary top-$b_t$ selection. Equal scores preserve candidate order. Set $S_\pm=\mathcal T_{\mathcal F}(\pm I)$. Calibration holds the model, candidates, quotas, and $I,H$ fixed, retaining the full correction if feasible or returning the stored subset at a verified lower endpoint after ten bisections. Stored utility values are promoted to double precision before interpolation and gap comparisons; the mean sums use a common index order. The certificate concerns these computed values. Appendix~\ref{app:utility-budget} gives the cap, search, and boundary cases.

\subsection{Selection loop}
For each seed, the common warm set $W$ is drawn without replacement and clipped to $B$. Adaptive selectors train on $W$, then take their reference-pretraining updates. Development chooses either retention, counting $W$ toward $B$, or reselection, starting $S$ empty with the entire pool eligible. A released warm example counts only if selected again; retained $W$ filling $B$ ends selection. Algorithm~\ref{alg:dvc} initializes the previous snapshot to the current model.

Each round makes $\lceil\min(m,|\mathcal P\setminus S|)/16\rceil$ queries. Groups are clipped to 16, source availability, and remaining allowance, so short groups may yield fewer than $m$ candidates. The retained batch is clipped to candidate count and remaining budget. Retained examples cannot recur; unretained candidates may recur in later rounds.

\begin{algorithm}[htbp]
\caption{Hierarchical multi-source selection}
\label{alg:dvc}
\begin{algorithmic}[1]
\Require Pool $\mathcal P$, nonempty reference $V$, $C\geq2$
\Require Integer budget $1\leq B\leq|\mathcal P|$; positive integers $m,b$
\Require $\lambda_{\max}\geq0$, $\rho\in[0,1]$
\State Sample warm set $W\subseteq\mathcal P$, $1\leq|W|\leq B$
\State Initialize selector on $W$; apply reference pretraining
\State $S\gets W$ if retained; otherwise $S\gets\varnothing$
\State Set previous snapshot $\theta^-\gets\theta$; initialize source weights
\While{$|S|<B$}
 \State Compute current reference gradient and features
 \State $\mathcal C\gets\varnothing$
 \For{$q=1,\ldots,\lceil\min(m,|\mathcal P\setminus S|)/16\rceil$}
  \State Identify sources with examples outside $S\cup\mathcal C$
  \State Restart equal source weights if availability changed
  \State Sample source according to EXP3 probabilities
  \State Query at most 16 new candidates from that source
  \State Compute and return their mean bounded reward
  \State Add the queried candidates to $\mathcal C$
 \EndFor
 \State $b_t\gets\min(b,B-|S|,|\mathcal C|)$
 \State Compute $I,\bar Q,\bar R,\bar D,U,T$ at the current state
 \State $H_i\gets\bar Q_i(\bar R_i+\bar D_i+U_i+T_i)/4$
 \State Fix $\mathcal F$ from class quotas if enabled
 \State Compute $S_+,S_-$ and range $R_{\mathcal F}$
 \State Calibrate $\hat t$; store its evaluated subset $\hat S$
 \State $S\gets S\cup\hat S$
 \State $\theta^-\gets\theta$; update selector on accumulated $S$
\EndWhile
\State Train a fresh classifier on $S$
\State \Return $S$ and the final classifier
\end{algorithmic}
\end{algorithm}

The reference gradient and reward scale $s_t=\max\{10^{-12},\operatorname{median}_{z\in V}|\langle g_V,g_z\rangle|\}$ stay fixed throughout each round and are shared across sources. Each group returns its mean $(1+\tanh(a_i/s_t))/2$ over all queried candidates, including those not retained.

If $K_0$ is the initial number of sources, the estimated remaining query horizon is
\begin{equation}
 T_0=\max\left\{1,
 \left\lceil\frac{B-|S|}{b}\right\rceil
 \left\lceil\frac{m}{16}\right\rceil\right\}.
\end{equation}
The source policy uses
\begin{align}
 \gamma&=\min\left\{1,
 \sqrt{\frac{K_0\log(\max\{K_0,2\})}{(\mathrm e-1)T_0}}\right\},\\
 \eta&=\gamma/K_0.
\end{align}
Uniform exploration is distributed over currently available sources. When this set changes, equal weights are restored and $T_0$ is recomputed from the remaining budget. A single available source is chosen with probability one. The estimated horizon determines exploration; it need not equal the eventual query count. Appendix~\ref{app:bandit} describes the bounded feedback and its conditional unbiasedness.

\subsection{Coverage and ablation conventions}
Let $n_c$ be the number of already selected examples with observed label $c$ and let $\rho_c=|V_c|/|V|$. For a batch of size $b_t$, the desired cumulative count is $\rho_c(|S|+b_t)$. Starting from zero quotas, allocate each position to the class with the largest difference between its desired count and $n_c+q_c$, provided its candidate capacity has not been reached. Ties follow class index. Selecting the top $q_c$ scores in each class then solves the fixed-quota objective exactly. When coverage is disabled, ordinary top-$b_t$ ranking solves the cardinality-only objective.

Removing quality replaces its multiplicative gate by one. Removing relevance, diversity, uncertainty, or stability drops that term and divides by the number of remaining terms. These variants retain the same maximum correction strength and utility tolerance. The impact ablation ranks by the gated auxiliary average and disables the utility constraint, so impact does not remain in the within-batch ranking through that constraint. Source queries still use the common gradient-based feedback; this ablation isolates the impact term in batch selection. Removing the auxiliary correction leaves the gradient anchor; removing only the utility constraint uses the full correction without calibration. The uniform-source ablation retains the complete scoring and calibration rules while sampling candidates uniformly from the remaining pool. The flat-score control averages the five auxiliary signals and $(I+1)/2$ with equal weight, without the hierarchical gate or utility calibration.

The input-view control computes quality and diversity from $h_0=\phi(x)$ and relevance from the first affine block. The hidden-view control uses $h_1$ and the second affine block. Both retain the full-parameter impact, prediction entropy, temporal stability, and utility calibration. These controls test the two local views under the same global evidence.

\subsection{Complexity}
\label{app:complexity}
Let $P$ be the parameter count, $m$ the candidate limit, $v=|V|$, $M$ the memory size, and $H=d+h+C$ the sum of input, hidden, and output widths. Forward passes and exact gradient contractions, including both model snapshots, cost $O((m+v+M)P)$ per round. Source feedback evaluates the queried groups before batch scoring and adds $O(mP)$ work, leaving the same asymptotic bound. Candidate-to-reference quality and candidate-to-memory novelty add $O(m(v+M)H)$. Ordering all reference similarities gives a conservative $O(mv\log v)$ bound for neighbor retrieval; partial top-$k$ selection may use less work. Weighted class agreement adds $O(mk'C)$, with $k'=\min\{k,v\}$. Thus a scoring-round upper bound is
\begin{equation}
 \begin{split}
 O\big(& (m+v+M)P+m(v+M)H\\
       & {}+mv\log v+mk'C\big).
 \end{split}
\end{equation}
Calibration uses at most 13 rankings: two utility endpoints, the full correction, and ten midpoints. A ranking with quotas costs $O(m\log m+C(m+b_t)+|S|+v)$ for the implemented class scans and quota allocation, without additional gradients. Scoring storage is $O(P+(m+v+M)H+m(v+M)+mk'C)$; it avoids a candidate-by-parameter gradient matrix. Availability masks over $N$ source examples add $O(KN\lceil m/16\rceil)$ time per round and $O(N)$ storage. The resident source features occupy $O(Nd)$ memory. Feature construction, selector updates, and final training are additional costs. Fixed candidate limits reduce scoring work, but the complete pipeline is not sublinear in $N$.

\FloatBarrier

\section{Experimental Details and Additional Results}
\label{app:experiments}
\subsection{Data and preprocessing}
\label{app:data}
TREC uses its six coarse categories and standard train/test files. Sentence polarity uses a stratified 80/20 split fixed before development. For 20 Newsgroups, the standard split excludes headers, footers, and quoted replies to reduce formatting and author cues. AG News retains its standard four-way partition. BANKING77 uses the official 77-intent split: removing four repeated training texts and seven train--test matches leaves 9,992 training and 3,080 test examples, with a training-side pool of 9,224, reference set of 512, and development set of 256. Its four corruption sources follow the other classification tasks. These text corpora exclude empty texts, training duplicates after case and whitespace normalization, and training texts exactly matching test texts. Test membership is otherwise preserved; filtering uses text identity, never test labels.

Training-side features use at most 12,000 unigrams and bigrams, minimum document frequency two, and sublinear term frequency. IDF and 512-component truncated SVD are fitted on that side and applied unchanged to test texts. SVD uses seven power iterations and a fixed seed; dense vectors have unit Euclidean norm. The feature map includes unlabeled reference and development texts; their labels serve only their designated roles.

Amazon provides unigram and bigram counts without token order. Before target partitioning, we remove duplicate lexical records, including entire groups with contradictory labels: 146 of the original 8,000 records. No duplicate group spans domains. Vocabulary selection by document frequency, IDF, and SVD use only the three source domains; the held-out reference, development, and test sets use this source-fitted map and the provided counts.

\begin{table*}[t]
\centering
\small
\setlength{\tabcolsep}{4pt}
\begin{tabular}{lrrrrr}
\toprule
Task & Classes & Pool & Ref. & Dev. & Test \\
\midrule
TREC & 6 & 4,601 & 512 & 256 & 500 \\
Polarity & 2 & 7,778 & 512 & 256 & 2,139 \\
20 News & 20 & 10,189 & 512 & 256 & 7,317 \\
AG News & 4 & 119,071 & 512 & 256 & 7,600 \\
Books & 2 & 5,898 & 256 & 256 & 1,444 \\
DVD & 2 & 5,900 & 256 & 256 & 1,442 \\
Electronics & 2 & 5,888 & 256 & 256 & 1,454 \\
Kitchen & 2 & 5,876 & 256 & 256 & 1,466 \\
\bottomrule
\end{tabular}
\caption{Dataset sizes after text cleaning and duplicate removal. The Amazon row names the held-out target domain; its pool consists of the other three domains.}
\label{tab:data}
\end{table*}

\subsection{Primary comparison with shared final training}
\label{app:fair}
\begin{table*}[t]
\centering
\footnotesize
\setlength{\tabcolsep}{2pt}
\begin{tabular}{lrrrrrrrr}
\toprule
Method & TREC & Polarity & 20 News & AG News & Books & DVD & Electronics & Kitchen \\
\midrule
Random & 68.8 {\scriptsize$\pm$3.7} & 62.6 {\scriptsize$\pm$1.2} & 49.4 {\scriptsize$\pm$0.4} & 86.4 {\scriptsize$\pm$0.3} & 71.3 {\scriptsize$\pm$2.8} & 73.6 {\scriptsize$\pm$1.1} & 77.3 {\scriptsize$\pm$2.0} & 79.8 {\scriptsize$\pm$1.2} \\
Balanced random & 69.8 {\scriptsize$\pm$4.4} & 62.3 {\scriptsize$\pm$1.5} & 49.7 {\scriptsize$\pm$0.5} & 86.4 {\scriptsize$\pm$0.2} & 71.9 {\scriptsize$\pm$1.2} & 73.7 {\scriptsize$\pm$1.1} & 76.6 {\scriptsize$\pm$1.5} & 80.2 {\scriptsize$\pm$1.5} \\
Entropy & 70.1 {\scriptsize$\pm$2.0} & 62.1 {\scriptsize$\pm$1.4} & 47.9 {\scriptsize$\pm$0.9} & \textbf{87.2} {\scriptsize$\pm$0.2} & 69.8 {\scriptsize$\pm$1.7} & 73.0 {\scriptsize$\pm$1.3} & 78.0 {\scriptsize$\pm$1.4} & 80.5 {\scriptsize$\pm$1.4} \\
$k$-center & 70.5 {\scriptsize$\pm$2.1} & 61.8 {\scriptsize$\pm$1.8} & 48.2 {\scriptsize$\pm$1.2} & 86.5 {\scriptsize$\pm$0.3} & 71.0 {\scriptsize$\pm$1.7} & 74.7 {\scriptsize$\pm$0.7} & 77.9 {\scriptsize$\pm$1.3} & 81.6 {\scriptsize$\pm$0.5} \\
BADGE-style & 67.2 {\scriptsize$\pm$3.3} & 62.4 {\scriptsize$\pm$1.5} & 48.9 {\scriptsize$\pm$0.8} & 86.8 {\scriptsize$\pm$0.2} & 71.3 {\scriptsize$\pm$2.2} & 72.8 {\scriptsize$\pm$1.6} & 76.9 {\scriptsize$\pm$1.1} & 80.4 {\scriptsize$\pm$1.8} \\
EL2N & 46.7 {\scriptsize$\pm$6.6} & 59.2 {\scriptsize$\pm$1.3} & 39.1 {\scriptsize$\pm$0.7} & 36.5 {\scriptsize$\pm$1.1} & 70.7 {\scriptsize$\pm$1.6} & 70.8 {\scriptsize$\pm$2.3} & 73.6 {\scriptsize$\pm$1.8} & 79.7 {\scriptsize$\pm$1.3} \\
Alignment & 70.0 {\scriptsize$\pm$5.2} & 65.2 {\scriptsize$\pm$1.1} & 52.9 {\scriptsize$\pm$0.5} & 84.1 {\scriptsize$\pm$0.3} & 73.9 {\scriptsize$\pm$1.4} & \textbf{76.9} {\scriptsize$\pm$0.7} & 80.3 {\scriptsize$\pm$0.9} & 82.0 {\scriptsize$\pm$1.0} \\
Gradient cosine & 69.3 {\scriptsize$\pm$5.2} & 65.1 {\scriptsize$\pm$1.1} & 54.7 {\scriptsize$\pm$0.6} & 86.4 {\scriptsize$\pm$0.3} & 75.1 {\scriptsize$\pm$1.1} & 74.7 {\scriptsize$\pm$1.3} & 79.6 {\scriptsize$\pm$0.9} & 81.4 {\scriptsize$\pm$0.6} \\
Alignment + coverage & 65.8 {\scriptsize$\pm$4.5} & 64.5 {\scriptsize$\pm$1.4} & 55.1 {\scriptsize$\pm$0.6} & 82.8 {\scriptsize$\pm$0.3} & 74.4 {\scriptsize$\pm$1.3} & 76.2 {\scriptsize$\pm$0.6} & 80.3 {\scriptsize$\pm$1.1} & 82.2 {\scriptsize$\pm$1.1} \\
GradMatch-style & 71.6 {\scriptsize$\pm$3.5} & 63.4 {\scriptsize$\pm$1.3} & 55.7 {\scriptsize$\pm$0.7} & 75.2 {\scriptsize$\pm$0.8} & 74.9 {\scriptsize$\pm$1.2} & 75.7 {\scriptsize$\pm$1.1} & 79.6 {\scriptsize$\pm$1.4} & 81.6 {\scriptsize$\pm$1.2} \\
GLISTER-style & 75.5 {\scriptsize$\pm$2.3} & 65.4 {\scriptsize$\pm$0.8} & 55.8 {\scriptsize$\pm$0.7} & 86.9 {\scriptsize$\pm$0.2} & 74.7 {\scriptsize$\pm$1.6} & 76.6 {\scriptsize$\pm$1.2} & 79.4 {\scriptsize$\pm$1.4} & 82.0 {\scriptsize$\pm$1.0} \\
Quality only & 66.5 {\scriptsize$\pm$3.9} & 64.9 {\scriptsize$\pm$0.9} & 55.1 {\scriptsize$\pm$0.7} & 80.9 {\scriptsize$\pm$0.3} & 74.6 {\scriptsize$\pm$1.7} & 75.9 {\scriptsize$\pm$1.1} & 79.3 {\scriptsize$\pm$0.9} & 79.8 {\scriptsize$\pm$0.8} \\
Flat six-signal & 70.6 {\scriptsize$\pm$4.0} & \textbf{65.9} {\scriptsize$\pm$0.8} & \textbf{57.2} {\scriptsize$\pm$0.3} & 86.4 {\scriptsize$\pm$0.2} & 73.2 {\scriptsize$\pm$0.9} & 76.0 {\scriptsize$\pm$1.1} & \textbf{81.1} {\scriptsize$\pm$1.0} & 81.4 {\scriptsize$\pm$0.9} \\
RHO-LOSS & 75.6 {\scriptsize$\pm$3.0} & 64.5 {\scriptsize$\pm$0.7} & 57.0 {\scriptsize$\pm$0.3} & 76.6 {\scriptsize$\pm$0.3} & \textbf{75.5} {\scriptsize$\pm$0.7} & 75.9 {\scriptsize$\pm$0.8} & 80.1 {\scriptsize$\pm$1.0} & 81.9 {\scriptsize$\pm$0.9} \\
Alignment + quality & 77.8 {\scriptsize$\pm$1.9} & 65.1 {\scriptsize$\pm$0.8} & 56.1 {\scriptsize$\pm$0.4} & 86.2 {\scriptsize$\pm$0.5} & 75.1 {\scriptsize$\pm$1.7} & 75.7 {\scriptsize$\pm$1.1} & 80.4 {\scriptsize$\pm$1.1} & 81.6 {\scriptsize$\pm$1.4} \\
Calibrated quality & 77.8 {\scriptsize$\pm$1.9} & 65.1 {\scriptsize$\pm$0.8} & 56.1 {\scriptsize$\pm$0.4} & 86.0 {\scriptsize$\pm$0.3} & 74.0 {\scriptsize$\pm$1.2} & 75.8 {\scriptsize$\pm$0.8} & 80.6 {\scriptsize$\pm$1.1} & 81.8 {\scriptsize$\pm$1.0} \\
DVC & \textbf{81.1} {\scriptsize$\pm$1.6} & \textbf{65.9} {\scriptsize$\pm$1.2} & 56.3 {\scriptsize$\pm$0.3} & 86.1 {\scriptsize$\pm$0.4} & 74.5 {\scriptsize$\pm$1.1} & 76.2 {\scriptsize$\pm$0.8} & 80.8 {\scriptsize$\pm$1.4} & \textbf{82.4} {\scriptsize$\pm$0.6} \\
Full pool & 83.0 {\scriptsize$\pm$1.5} & 69.5 {\scriptsize$\pm$0.9} & 59.7 {\scriptsize$\pm$0.4} & 87.3 {\scriptsize$\pm$0.2} & 76.4 {\scriptsize$\pm$0.4} & 78.6 {\scriptsize$\pm$0.7} & 82.4 {\scriptsize$\pm$1.0} & 85.8 {\scriptsize$\pm$0.5} \\
Reference only & 67.2 {\scriptsize$\pm$1.3} & 63.0 {\scriptsize$\pm$0.3} & 43.6 {\scriptsize$\pm$0.3} & 66.9 {\scriptsize$\pm$0.2} & 72.6 {\scriptsize$\pm$0.5} & 74.6 {\scriptsize$\pm$0.4} & 76.2 {\scriptsize$\pm$0.3} & 76.9 {\scriptsize$\pm$0.2} \\
\bottomrule
\end{tabular}
\caption{Comparison at 20\% source retention over ten matched seeds. All methods share the final optimizer, learning rate, class weighting, and update count within each task. Full-pool and reference-only controls use different training resources and do not enter subset rankings. Entries give per-task macro F1 mean $\pm$ sample standard deviation; bold indicates the best subset mean at the displayed precision.}
\label{tab:fair-full}
\end{table*}

\begin{table}[tbp]
\centering
\small
\setlength{\tabcolsep}{4pt}
\begin{tabular}{lrrrrr}
\toprule
Task & Width & Steps & LR & Decay & Power \\
\midrule
TREC & 64 & 400 & 0.001 & 0.1 & 0.5 \\
Polarity & 64 & 400 & 0.001 & 0.1 & 1 \\
20 News & 64 & 800 & 0.001 & 0.1 & 0.5 \\
AG News & 64 & 800 & 0.001 & 0.1 & 0.5 \\
Books & 64 & 400 & 0.003 & 0.1 & 0 \\
DVD & 64 & 200 & 0.001 & 0.1 & 0 \\
Electronics & 64 & 400 & 0.003 & 0.1 & 0 \\
Kitchen & 64 & 400 & 0.003 & 0.1 & 0 \\
\bottomrule
\end{tabular}
\caption{Common final-training protocol for the primary comparison. AdamW, minibatch size 64, and these settings apply to every subset, full-pool, and reference-only method. Power denotes inverse-class-frequency weighting; GradMatch also preserves its example weights.}
\label{tab:fair-recipes}
\end{table}

\begin{table}[tbp]
\centering
\small
\setlength{\tabcolsep}{4pt}
\begin{tabular}{lrrrr}
\toprule
Task & $\rho$ & Active (\%) & Mean $t$ & Max. loss (\%) \\
\midrule
TREC & 0.5 & 0.0 & 1.000 & 39.2 \\
Polarity & 0.5 & 0.0 & 1.000 & 39.5 \\
20 News & 0.5 & 0.0 & 1.000 & 18.4 \\
AG News & 0.025 & 100.0 & 0.126 & 2.5 \\
Books & 0.025 & 100.0 & 0.141 & 2.5 \\
DVD & 0.5 & 1.6 & 0.999 & 49.9 \\
Electronics & 0.25 & 74.3 & 0.650 & 25.0 \\
Kitchen & 0.25 & 51.0 & 0.752 & 25.0 \\
\bottomrule
\end{tabular}
\caption{Calibration in the primary comparison over ten matched seeds, using the selection trajectories underlying Table~\ref{tab:fair-main}. Active is the percentage of rounds with $t<1$, averaged across runs; mean $t$ first averages rounds within each run. Max. loss is the largest computed utility gap across all rounds and seeds, as a percentage of the feasible utility range.}
\label{tab:primary-calibration}
\end{table}

The primary comparison evaluates seventeen subset selectors, full-pool training, and reference-only training on the same prepared features and ten matched seeds. Every final classifier, including both resource controls, uses the task's shared settings in Table~\ref{tab:fair-recipes}. Learning rate, update count, weight decay, and the class-weighting rule are shared; class frequencies are computed separately from each method's actual final training set. Retained source examples are not augmented with reference examples. Macro F1 is the primary metric. The four Amazon domains contribute four of the eight tasks, so the task average summarizes these fixed benchmarks. Intervals average tasks within each seed before computing a paired Student-$t$ interval. The method-specific comparison separately chooses each method's final learning rate and class-weighting power. Table~\ref{tab:protocol-index} distinguishes the experiment families; paired differences are computed within a family.

\begin{table*}[t]
\centering\small
\setlength{\tabcolsep}{4pt}
\begin{tabular}{p{0.22\textwidth}p{0.64\textwidth}r}
\toprule
Comparison & Subset and final-training protocol & Seeds \\
\midrule
Primary MLP & Source-only subsets; shared task-specific final-training protocol; seventeen selectors. & 10 \\
Method-specific MLP & Source-only subsets; separately chosen final learning rate and weighting power; 20\% retention / other budgets. & 10 / 5 \\
Primary CNN transfer & The primary MLP subsets held fixed; shared CNN settings, without reference examples. & 10 \\
Independent scoring & Full/quality evidence with/without calibration and compact evidence; equal development search; shared final-training protocol and sample-ID training order. & 10 \\
Nested reference & Fixed source pool; three 128-example reference compositions within a 512-example reference set. & 5 \\
Components and tolerance & Fixed selection settings; shared final MLP training; complete selection trajectories. & 5 \\
Fixed-state utility & Common model, candidates, and quotas; local scoring paths and subsequent updates, rather than new final classifiers. & 5 \\
\bottomrule
\end{tabular}
\caption{Experimental protocols. Seed counts refer to evaluation, not development. All comparisons use disjoint development seeds. The independently tuned scoring controls retain common final training and use a separate paired evaluation.}
\label{tab:protocol-index}
\end{table*}

\subsection{Sensitivity to benchmark weighting}
\label{app:weighting}
\begin{table*}[t]
\centering\small
\begin{tabular}{lrr}
\toprule
Benchmark weighting & vs. Alignment + quality & vs. Calibrated quality \\
\midrule
Eight tasks & 0.7 [0.3, 1.1] & 0.8 [0.5, 1.1] \\
Seven tasks, excluding TREC & 0.3 [-0.1, 0.7] & 0.4 [0.1, 0.7] \\
Five corpora & 0.9 [0.5, 1.3] & 1.0 [0.6, 1.3] \\
Four corpora, excluding TREC & 0.3 [0.0, 0.7] & 0.4 [0.2, 0.7] \\
\bottomrule
\end{tabular}
\caption{Descriptive sensitivity of DVC gains in the primary comparison. Entries are macro-F1 point differences and paired-seed 95\% intervals. Each interval aggregates tasks within each of ten seeds before estimating variation. Corpus weighting first averages the four Amazon domains, then weights that mean equally with each other corpus. These exploratory summaries concern the fixed benchmarks and do not estimate uncertainty over unseen tasks.}
\label{tab:primary-sensitivity}
\end{table*}

Four Amazon domains account for half of the eight-task mean. We therefore also average these domains before giving each of the five corpora equal weight. DVC's gain is then 0.9 points over alignment plus quality and 1.0 over calibrated quality. Excluding TREC, which has the largest gain, leaves seven-task differences of 0.3 and 0.4 points. The latter interval remains positive; the former includes zero. These post hoc summaries describe sensitivity to the weighting of these fixed benchmarks, with methods and predictions unchanged.

\subsection{Training and model selection}
The initial subset contains 256 uniformly sampled source examples, clipped to the total budget. The baseline selection search compares mandatory retention with releasing these examples after initialization; released examples count only if selected again. Learned selectors receive ten AdamW updates (learning rate $0.003$, weight decay $10^{-4}$), followed by reference pretraining. Other methods bypass these updates. Each subsequent round uses five SGD updates (learning rate $0.2$, batch size 64) on the accumulated subset. Source query groups contain 16 examples; retained selection batches contain 32.

A freshly initialized MLP trains with AdamW \citep{loshchilov2019adamw}, excluding reference and development examples except in the labeled reference-only controls. The primary comparison shares Table~\ref{tab:fair-recipes}'s final-training settings; its weight decay is fixed at $0.1$. The supplementary comparison separately selects learning rate and class-weighting power. For a final training set $S$, let $n_c$ be its observed-label count, $p$ the weighting power, and $\beta_i\geq0$ the base example weight (one, except for GradMatch). We use
\begin{align}
 \widetilde w_i&= \beta_i\max\{1,n_{y_i}\}^{-p},\nonumber\\
 w_i&=\begin{cases}
 \displaystyle\frac{\widetilde w_i}{|S|^{-1}\sum_{j\in S}\widetilde w_j},
       &\sum_{j\in S}\widetilde w_j>0,\\[3pt]
 1, &\text{otherwise}.
 \end{cases}
\label{eq:training-weights}
\end{align}
The zero-total-weight case has a uniform-weight extension. Otherwise the common rule yields weights determined by the selected labels. Empty classes contain no training example; clamping their count only prevents an undefined denominator. GradMatch's base weights are multiplied by the frequency factor before mean normalization. In the method-specific comparison, reference-only training uses 400 updates on 20 Newsgroups and AG News, versus 800 for source subsets. Models use PyTorch \citep{paszke2019pytorch} and preprocessing uses scikit-learn \citep{pedregosa2011sklearn}. The CNN learns token embeddings from scratch with three convolution widths.

\begin{table*}[t]
\centering
\footnotesize
\setlength{\tabcolsep}{2pt}
\begin{tabular}{lrrrrrr}
\toprule
Task & $k$ & $\lambda_{\max}$ & $\rho$ & Ref. init. & Quotas & Initial set \\
\midrule
TREC & 4 & 2 & 0.5 & 50 & yes & release \\
Polarity & 32 & 2 & 0.5 & 50 & yes & keep \\
20 News & 32 & 2 & 0.5 & 50 & yes & keep \\
AG News & 32 & 4 & 0.025 & 50 & yes & keep \\
Books & 8 & 4 & 0.025 & 0 & yes & release \\
DVD & 8 & 4 & 0.5 & 20 & yes & release \\
Electronics & 8 & 4 & 0.25 & 20 & yes & release \\
Kitchen & 8 & 4 & 0.25 & 20 & yes & keep \\
\bottomrule
\end{tabular}
\caption{DVC selection settings, fixed across retention budgets. Reference initialization is measured in gradient updates.}
\label{tab:selection-settings}
\end{table*}

\paragraph{Shared settings and development.}
The neighborhood softmax uses $\tau=10$ in both views; ten bisection steps calibrate the auxiliary coefficient. The final MLP weight decay is fixed at $0.1$ in the shared settings. Final learning rates are selected from $\{0.001,0.003\}$ and class-weighting powers from $\{0,0.5,1\}$ on TREC, Polarity, and 20 Newsgroups; Amazon uses power zero. Development learning curves determine update counts, considering $\{200,400,800\}$ on Amazon. Table~\ref{tab:fair-recipes} specifies the configuration applied to every method, including AG News. Class weighting follows Equation~\ref{eq:training-weights} rather than a fixed vector shared across subsets.

The staged selection search considers $\lambda_{\max}\in\{0.25,0.5,1,2,4\}$, $k\in\{4,8,16,32\}$, reference pretraining in $\{0,20,50\}$ updates, class quotas on/off, and initial-set retention/release. Its fourteen selectors are random, balanced random, entropy, EL2N, $k$-center, BADGE, raw alignment, gradient cosine, quota-constrained alignment, GradMatch-style, GLISTER-style, quality-only, flat six-signal, and DVC. Learned comparison methods share the reference-pretraining choices; all fourteen compare retention/release on development seeds 101, 103, and 107. Utility budgets $\rho\in\{0.025,0.1,0.25,0.5,1\}$ are combined with the coefficient grid in stages, rather than by exhaustive Cartesian search. The fixed-coefficient control searches all five coefficients with $\rho=1$.

RHO-LOSS, alignment plus quality, and calibrated quality extend this set to seventeen selectors. Appendix~\ref{app:supplement} distinguishes their shared-setting development choices from the independently tuned comparison. Selection settings are held fixed across retention budgets.

\paragraph{CNN training.}
The CNN uses 100-dimensional random token embeddings, 64 filters per width 3, 4, and 5, max pooling, and dropout 0.3. Its training-side vocabulary has at most 20,000 entries, including padding and unknown tokens. Maximum sequence length is 64 on TREC and 128 otherwise. AdamW uses learning rate 0.001, weight decay 0.01, and batch size 64; update counts are 600, 1,200, 2,400, and 2,400 for TREC, Polarity, 20 Newsgroups, and AG News. AG News transfers the 20 Newsgroups settings. Only TREC uses development-selected inverse-square-root class-frequency weighting. Frequencies are computed from each retained training set. Both architectures preserve GradMatch's base example weights when applying class weighting.

Development seeds are 101, 103, and 107. The primary 20\% comparison, representation-view controls, primary CNN transfer, method-specific 20\% comparisons, fixed-coefficient comparison, and BANKING77 extension use seeds 11, 22, 33, 44, 55, 211, 223, 227, 229, and 233; component and tolerance comparisons use the first five. Independently tuned controls use these ten evaluation seeds in their own paired block. Seeds control initialization, corruption, source assignment, candidate queries, and minibatch draws; partitions remain fixed. Selection settings chosen at 20\% retention remain fixed across budgets. BANKING77 transfers each method's complete 20 Newsgroups settings, including retention, selection, and final training, without BANKING77 development selection. It tests transfer of settings across both domain and label space.

\subsection{Baseline definitions}
\label{app:baselines}
\paragraph{Random and class-balanced random.}
Random retains a uniform subset, including initial examples only when retention is enabled. Balanced random applies \DVC's class-deficit rule to independent random scores over the remaining pool. Neither queries target gradients.

\paragraph{Entropy and EL2N.}
Entropy ranks by current output entropy; EL2N uses the Euclidean distance between predicted probabilities and the observed one-hot label \citep{paul2021diet}. Both share the evolving selector and subset budget, without requesting new annotations.

\paragraph{Geometric coverage and BADGE.}
The $k$-center control greedily maximizes minimum cosine distance from retained memory and that round's selected examples, using normalized lexical features \citep{sener2018coreset}. BADGE uses pseudo-label last-layer gradients, initializes with the largest gradient norm, then samples by squared distance \citep{ash2020badge}. It excludes selected centers and uses uniform sampling among unselected candidates when all distances vanish.

\paragraph{Gradient controls.}
Raw alignment ranks by $\langle g_V,g_i\rangle$; gradient cosine divides by $\|g_i\|$, since the common $\|g_V\|$ leaves rankings unchanged. The class-coverage control maximizes alignment under \DVC's quotas. All use exact full-parameter contractions, including input factors and biases.

\paragraph{Gradient matching.}
GradMatch-style selection matches last-layer per-example gradients to the reference gradient \citep{pmlr-v139-killamsetty21a}. Each greedy step adds the candidate with largest normalized residual correlation, fits nonnegative least-squares coefficients, and updates the residual. Coefficients with positive total weight are normalized to mean one within the selected batch; the all-zero case uses the uniform-weight extension in Equation~\ref{eq:training-weights}. Initial examples retain weight one.

\paragraph{Taylor-greedy validation selection.}
GLISTER-style selection uses candidate last-layer gradient inner products with a reference gradient recomputed at a virtual model state \citep{killamsetty2021glister}. Each selected candidate takes a virtual last-layer step of $0.2/b$; candidate gradients stay fixed within the round. Real updates use the accumulated subset. This Taylor approximation adapts GLISTER to the restricted query budget.

\subsection{Evaluation and uncertainty}
We report macro F1, accuracy, and matched-seed means and sample standard deviations; tables give run counts. Performance scores and their uncertainty are displayed to one decimal place. Paired differences and intervals are calculated before display rounding. Methods share each seed's corrupted pool and final initialization. Paired intervals use Student $t$ on seed-wise differences after averaging the eight fixed tasks: nine degrees of freedom for ten seeds, four for five. They describe training and selection variation on these tasks, not uncertainty over datasets. Comparisons preserve seed matching \citep{dror-etal-2018-hitchhikers}; positive estimates alone do not establish significance.

Controlled corruption isolates source-label quality, while natural transfer examines domain shift. Class coverage, retained corruption, reference-size sensitivity, and ablations complement the aggregate metrics.

\subsection{Transferred settings on BANKING77}
\begin{table}[tbp]
\centering
\small
\setlength{\tabcolsep}{4pt}
\begin{tabular}{lrr}
\toprule
Method & Macro F1 & Accuracy \\
\midrule
Random & 69.6 {\scriptsize$\pm$0.6} & 70.0 {\scriptsize$\pm$0.6} \\
Balanced random & 70.0 {\scriptsize$\pm$1.0} & 70.4 {\scriptsize$\pm$1.0} \\
Entropy & 65.8 {\scriptsize$\pm$1.7} & 66.5 {\scriptsize$\pm$1.5} \\
$k$-center & 71.5 {\scriptsize$\pm$0.8} & 71.8 {\scriptsize$\pm$0.8} \\
BADGE-style & 68.7 {\scriptsize$\pm$1.1} & 69.1 {\scriptsize$\pm$0.9} \\
EL2N & 58.7 {\scriptsize$\pm$1.5} & 59.9 {\scriptsize$\pm$1.3} \\
Alignment & 69.8 {\scriptsize$\pm$1.0} & 70.7 {\scriptsize$\pm$0.9} \\
Gradient cosine & 69.6 {\scriptsize$\pm$1.3} & 70.0 {\scriptsize$\pm$1.2} \\
Alignment + coverage & 70.8 {\scriptsize$\pm$0.9} & 71.5 {\scriptsize$\pm$0.8} \\
GradMatch-style & 62.2 {\scriptsize$\pm$1.1} & 64.3 {\scriptsize$\pm$0.9} \\
GLISTER-style & 59.6 {\scriptsize$\pm$2.1} & 62.6 {\scriptsize$\pm$1.6} \\
Quality only & 72.7 {\scriptsize$\pm$0.7} & 73.0 {\scriptsize$\pm$0.6} \\
Flat six-signal & 74.8 {\scriptsize$\pm$1.0} & 74.7 {\scriptsize$\pm$1.0} \\
\midrule
DVC & 74.1 {\scriptsize$\pm$0.6} & 74.1 {\scriptsize$\pm$0.5} \\
\midrule
Full pool & 79.3 {\scriptsize$\pm$0.6} & 80.2 {\scriptsize$\pm$0.5} \\
Reference only & 57.4 {\scriptsize$\pm$0.4} & 57.5 {\scriptsize$\pm$0.4} \\
\bottomrule
\end{tabular}
\caption{Prospective BANKING77 evaluation at a 20\% source budget (10 seeds per setting). Entries give mean and sample standard deviation in percent. Each method uses its 20 Newsgroups settings without BANKING77 tuning. Full-pool and reference-only controls have different training resources.}
\label{tab:banking77}
\end{table}

BANKING77 shares the classification tasks' retention budget and reference size; its transferred settings are evaluated separately from the eight-task average.

\subsection{Method-specific training settings and budget comparisons}
\begin{figure*}[tbp]
\centering\includegraphics[width=\textwidth]{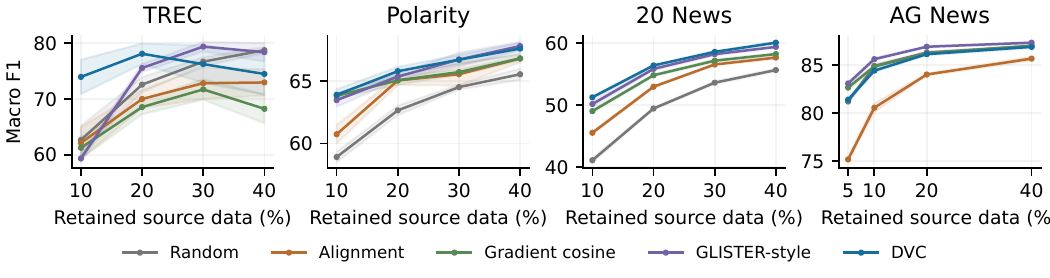}
\caption{MLP macro F1 as the retained source budget increases. Lines show seed means; shaded bands show one standard error computed with the observed seed count. The 20\% points use 10 seeds per setting; other budgets use 5 seeds per setting. AG News starts at 5\%.}
\label{fig:budgets}
\end{figure*}

\begin{table*}[t]
\centering
\small
\setlength{\tabcolsep}{4pt}
\begin{tabular}{lrrrr}
\toprule
Selector & \shortstack{10\% / 5\%} & 20\% & \shortstack{30\% / 20\%} & 40\% \\
\midrule
Random & 66.1 {\scriptsize$\pm$0.6} & 71.6 {\scriptsize$\pm$0.7} & 74.1 {\scriptsize$\pm$0.2} & 75.4 {\scriptsize$\pm$0.3} \\
Balanced random & 65.3 {\scriptsize$\pm$0.9} & 71.1 {\scriptsize$\pm$0.4} & 73.3 {\scriptsize$\pm$0.3} & 74.4 {\scriptsize$\pm$0.7} \\
Entropy & 66.2 {\scriptsize$\pm$1.1} & 71.5 {\scriptsize$\pm$0.4} & 74.3 {\scriptsize$\pm$0.7} & 75.8 {\scriptsize$\pm$0.4} \\
$k$-center & 66.3 {\scriptsize$\pm$1.3} & 71.7 {\scriptsize$\pm$0.6} & 73.5 {\scriptsize$\pm$0.5} & 74.6 {\scriptsize$\pm$1.0} \\
BADGE-style & 65.9 {\scriptsize$\pm$0.8} & 71.0 {\scriptsize$\pm$0.7} & 74.3 {\scriptsize$\pm$0.6} & 75.8 {\scriptsize$\pm$0.3} \\
EL2N & 57.0 {\scriptsize$\pm$1.2} & 59.7 {\scriptsize$\pm$0.8} & 59.6 {\scriptsize$\pm$0.4} & 65.9 {\scriptsize$\pm$1.3} \\
Alignment & 68.4 {\scriptsize$\pm$0.7} & 73.2 {\scriptsize$\pm$0.7} & 74.5 {\scriptsize$\pm$0.6} & 75.5 {\scriptsize$\pm$0.7} \\
Gradient cosine & 69.3 {\scriptsize$\pm$0.6} & 73.4 {\scriptsize$\pm$0.8} & 74.7 {\scriptsize$\pm$0.8} & 75.3 {\scriptsize$\pm$0.8} \\
Alignment + coverage & 68.3 {\scriptsize$\pm$0.4} & 72.6 {\scriptsize$\pm$0.6} & 74.7 {\scriptsize$\pm$0.4} & 75.2 {\scriptsize$\pm$0.7} \\
GradMatch-style & 68.5 {\scriptsize$\pm$0.8} & 72.0 {\scriptsize$\pm$0.4} & 73.2 {\scriptsize$\pm$0.7} & 73.6 {\scriptsize$\pm$0.5} \\
GLISTER-style & 69.3 {\scriptsize$\pm$0.3} & 74.4 {\scriptsize$\pm$0.5} & 75.7 {\scriptsize$\pm$0.4} & 75.8 {\scriptsize$\pm$0.8} \\
Quality only & 68.3 {\scriptsize$\pm$0.4} & 72.3 {\scriptsize$\pm$0.5} & 73.7 {\scriptsize$\pm$0.7} & 75.3 {\scriptsize$\pm$0.5} \\
Flat six-signal & 70.6 {\scriptsize$\pm$0.2} & 74.5 {\scriptsize$\pm$0.5} & 76.1 {\scriptsize$\pm$0.6} & 77.1 {\scriptsize$\pm$0.5} \\
\midrule
DVC & 71.9 {\scriptsize$\pm$0.8} & 75.2 {\scriptsize$\pm$1.0} & 76.0 {\scriptsize$\pm$0.8} & 76.7 {\scriptsize$\pm$0.9} \\
\bottomrule
\end{tabular}
\caption{Eight-task macro F1 under method-specific training settings. Means and sample standard deviations first average tasks within each matched seed. The 20\% column uses ten seeds; other columns use the five seeds common to every task. In split headers, the first budget applies to seven tasks and the second to AG News; unsplit headers apply to all eight. Per-task curves for five selectors appear in Figure~\ref{fig:budgets}; the horizontal axes give each task's actual retention budget.}
\label{tab:budget-summary}
\end{table*}

Table~\ref{tab:budget-summary} summarizes the eight-task trajectories under method-specific training settings. Selection choices remain fixed as retention increases; Figure~\ref{fig:budgets} separates the four text corpora. The method's relative advantage varies with the amount retained, and simple alternatives lead in some task--budget combinations. The primary comparison instead shares one task-specific final-training protocol across selectors (Table~\ref{tab:fair-full}).

\subsection{Independently tuned scoring controls}
\label{app:independent}
\begin{table*}[t]
\centering
\small
\setlength{\tabcolsep}{4pt}
\begin{tabular}{lrrrrr}
\toprule
Task & DVC & Full, fixed & Quality, budgeted & Quality, fixed & Compact, budgeted \\
\midrule
TREC & \textbf{80.5} {\scriptsize$\pm$1.5} & 80.4 {\scriptsize$\pm$1.0} & 77.0 {\scriptsize$\pm$1.4} & 78.9 {\scriptsize$\pm$1.9} & 78.9 {\scriptsize$\pm$2.3} \\
Polarity & 65.5 {\scriptsize$\pm$0.6} & \textbf{65.9} {\scriptsize$\pm$0.9} & 65.7 {\scriptsize$\pm$0.7} & 65.8 {\scriptsize$\pm$0.8} & 65.3 {\scriptsize$\pm$1.0} \\
20 News & 57.0 {\scriptsize$\pm$0.4} & \textbf{57.3} {\scriptsize$\pm$0.2} & 56.4 {\scriptsize$\pm$0.5} & 56.7 {\scriptsize$\pm$0.2} & 57.2 {\scriptsize$\pm$0.3} \\
AG News & \textbf{86.2} {\scriptsize$\pm$0.3} & 85.8 {\scriptsize$\pm$0.3} & 84.9 {\scriptsize$\pm$0.7} & 85.5 {\scriptsize$\pm$0.5} & 86.0 {\scriptsize$\pm$0.3} \\
Books & 74.2 {\scriptsize$\pm$1.2} & 74.0 {\scriptsize$\pm$1.1} & 74.6 {\scriptsize$\pm$1.0} & 73.9 {\scriptsize$\pm$1.6} & \textbf{74.9} {\scriptsize$\pm$1.2} \\
DVD & 75.8 {\scriptsize$\pm$1.2} & 75.8 {\scriptsize$\pm$1.2} & \textbf{77.1} {\scriptsize$\pm$1.1} & 75.5 {\scriptsize$\pm$1.6} & 76.2 {\scriptsize$\pm$1.0} \\
Electronics & 79.6 {\scriptsize$\pm$1.5} & 79.9 {\scriptsize$\pm$1.4} & 79.7 {\scriptsize$\pm$1.0} & 80.1 {\scriptsize$\pm$0.8} & \textbf{81.1} {\scriptsize$\pm$0.7} \\
Kitchen & 82.3 {\scriptsize$\pm$1.1} & 82.3 {\scriptsize$\pm$0.8} & 80.9 {\scriptsize$\pm$0.8} & \textbf{82.4} {\scriptsize$\pm$0.7} & 81.4 {\scriptsize$\pm$1.1} \\
8-task average & 75.1 {\scriptsize$\pm$0.3} & \textbf{75.2} {\scriptsize$\pm$0.4} & 74.5 {\scriptsize$\pm$0.3} & 74.8 {\scriptsize$\pm$0.4} & 75.1 {\scriptsize$\pm$0.5} \\
DVC gain [95\% CI] & -- & \shortstack{-0.1\\{\scriptsize$[-0.3,+0.2]$}} & \shortstack{+0.6\\{\scriptsize$[+0.2,+0.9]$}} & \shortstack{+0.3\\{\scriptsize$[0.0,+0.6]$}} & \shortstack{0.0\\{\scriptsize$[-0.4,+0.3]$}} \\
\bottomrule
\end{tabular}
\caption{Independently developed structure and calibration controls at 20\% retention over ten matched evaluation seeds. All families share the data resources and final-training protocol within each task. Entries are macro F1 mean $\pm$ sample standard deviation; bold indicates the best mean. The last row compares DVC with each family using paired seed-wise eight-task means. Intervals describe variation on these fixed tasks. Identical selected configurations share evaluations (DVD: DVC, Full, fixed).}
\label{tab:upgrade-matrix}
\end{table*}

\begin{table*}[t]
\centering
\small
\setlength{\tabcolsep}{4pt}
\begin{tabular}{lrrrrrrr}
\toprule
Task & $\rho$ (\%) & Active (\%) & Changed (\%) & Max. loss (\%) & Viol. & Fixed viol. (\%) & Select. (s) \\
\midrule
TREC & 50 & 1.4 & 1.4 & 50.0 & 0 & 1.4 & 1.0 \\
Polarity & 25 & 67.1 & 67.1 & 25.0 & 0 & 90.4 & 1.6 \\
20 News & 25 & 0.0 & 0.0 & 15.2 & 0 & 0.0 & 4.3 \\
AG News & 25 & 91.4 & 91.4 & 25.0 & 0 & 32.9 & 66.1 \\
Books & 25 & 38.3 & 38.3 & 25.0 & 0 & 62.8 & 0.9 \\
DVD & 100 & 0.0 & 0.0 & 43.0 & 0 & 0.0 & 1.0 \\
Electronics & 10 & 91.1 & 91.1 & 10.0 & 0 & 2.1 & 1.2 \\
Kitchen & 10 & 1.6 & 1.6 & 10.0 & 0 & 46.8 & 1.1 \\
\bottomrule
\end{tabular}
\caption{Actual calibration on the evaluation trajectories. Active and changed are the mean percentages of rounds with a reduced coefficient and a changed batch membership, respectively. Max. loss is the largest computed gap over every round and seed, as a percentage of the feasible utility range. Viol. counts all budget violations. Fixed viol. is the percentage of fixed-full rounds that exceed the corresponding DVC tolerance, measured on that method's own trajectory. Selection time is the median elapsed seconds across seeds and includes selector training.}
\label{tab:upgrade-calibration}
\end{table*}

\begin{table*}[t]
\centering\small
\begin{tabular}{lrrrrr}
\toprule
Task & Shared (\%) & $\Delta$ noise (pp) & $\Delta$ macro F1 [95\% CI] & Corrected & Introduced \\
\midrule
20 News & 52.9 & +0.2 & +0.6 [+0.3, +0.8] & 528.5 & 442.7 \\
AG News & 29.2 & -14.5 & +1.3 [+0.9, +1.7] & 338.3 & 243.6 \\
Books & 54.0 & -- & -0.5 [-1.5, +0.5] & 131.6 & 138.8 \\
DVD & 39.1 & -- & -1.3 [-2.5, 0.0] & 101.7 & 119.1 \\
Electronics & 40.6 & -- & -0.1 [-1.5, +1.2] & 134.9 & 136.7 \\
Kitchen & 35.3 & -- & +1.4 [+0.6, +2.2] & 146.7 & 126.0 \\
Polarity & 49.0 & +3.4 & -0.2 [-0.9, +0.4] & 170.4 & 175.4 \\
TREC & 50.7 & +0.8 & +3.4 [+2.4, +4.5] & 37.2 & 15.9 \\
\bottomrule
\end{tabular}
\caption{Endpoint comparison of independently tuned DVC and calibrated quality under the same final-training protocol. Shared is the proportion of retained examples common to both subsets. Differences are DVC minus calibrated quality; noise is the known synthetic corruption rate. Amazon has no synthetic corruption and its annotation accuracy is not known. Corrected and Introduced count test errors respectively repaired and introduced by DVC. Means and paired intervals use all matched evaluation seeds.}
\label{tab:upgrade-mechanisms}
\end{table*}

The controlled comparison independently selects five scoring variants using the development protocol in Appendix~\ref{app:supplement}. Table~\ref{tab:upgrade-matrix} reports all tasks, matched-seed variability, and paired aggregate differences. It uses source-only 20\% subsets, the shared final-training protocol, and canonical sample-ID training order. DVD selects the same full configuration with and without calibration, so those two cells share evaluations. The broad comparison and this controlled block are analyzed separately.

The complete hierarchy exceeds calibrated quality by 0.6 points, with gains whose task-level intervals exclude zero on TREC, 20 Newsgroups, AG News, and Kitchen. DVC gains 0.3 points over fixed quality, with an interval spanning zero. Full fixed and compact variants have similar aggregate F1. The comparison supports hierarchical evidence beyond quality-only correction, while calibration contributes the verified utility constraint. Averaging the two calibration settings gives a full-versus-quality family contrast of $+0.5$ ($[0.2,0.7]$); averaging the two structures gives a budgeted-versus-fixed contrast of $-0.2$ ($[-0.3,0.0]$). Their interaction is $+0.2$ ($[-0.3,0.7]$). These contrasts compare independently tuned families, not a single-module intervention with all other parameters fixed.

Table~\ref{tab:upgrade-calibration} checks the actual rounds. Every budgeted batch respects the computed allowance. Reduced coefficients change membership whenever calibration is active in these trajectories. The fixed-full violation percentage applies DVC's tolerance to the independently tuned fixed method on its own states; it is not a same-state causal estimate. Selection time includes initialization and online training and is distinct from final-model training.

\begin{figure*}[t]
\centering\includegraphics[width=\textwidth]{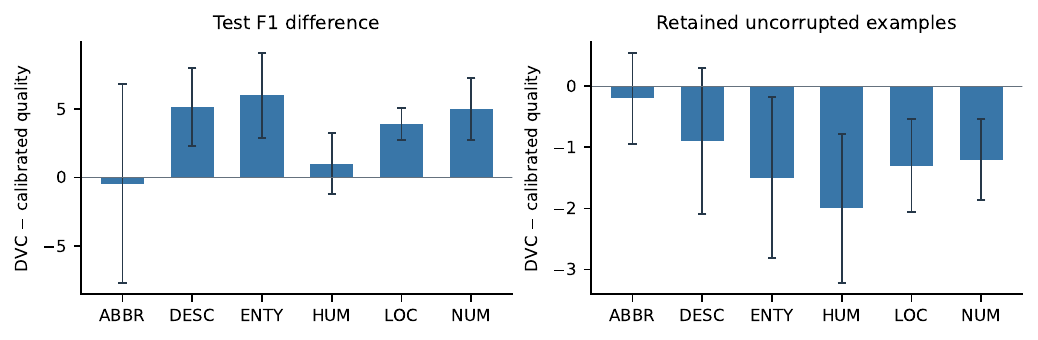}
\caption{TREC endpoint differences between independently tuned DVC and calibrated quality. Bars show matched-seed means; intervals are paired 95\% intervals over ten seeds. All six observed class counts are identical between the two subsets. Uncorrupted counts use synthetic corruption only for analysis.}
\label{fig:upgrade-mechanisms}
\end{figure*}

The endpoint diagnostics compare selected subsets and fresh-model predictions (Table~\ref{tab:upgrade-mechanisms}). On TREC, DVC corrects 37.2 quality-control errors and introduces 15.9 on average. Its F1 gains lie in DESC, ENTY, LOC, and NUM (3.9--6.0 points), while ABBR has a difference of $-0.5$ ($[-7.7,6.8]$). Observed class counts are identical, and DVC's corruption rate is 0.8 percentage points higher. The improvement therefore concerns the retained composition beyond total purity and class counts. On AG News, DVC instead reduces corruption by 14.5 points while gaining 1.3 F1. The component and view controls separately examine which parts of the hierarchy contribute to these outcomes.

\subsection{Representation-view controls}
\label{app:views}
\begin{table*}[t]
\centering
\footnotesize
\setlength{\tabcolsep}{2pt}
\begin{tabular}{lrrrrrrrr}
\toprule
Method & TREC & Polarity & 20 News & AG News & Books & DVD & Electronics & Kitchen \\
\midrule
Two views & \textbf{81.1} {\scriptsize$\pm$1.6} & \textbf{65.9} {\scriptsize$\pm$1.2} & 56.3 {\scriptsize$\pm$0.3} & 86.1 {\scriptsize$\pm$0.4} & \textbf{74.5} {\scriptsize$\pm$1.1} & \textbf{76.2} {\scriptsize$\pm$0.8} & 80.8 {\scriptsize$\pm$1.4} & 82.4 {\scriptsize$\pm$0.6} \\
Input only & 80.7 {\scriptsize$\pm$1.1} & 65.0 {\scriptsize$\pm$0.8} & 56.0 {\scriptsize$\pm$0.5} & 85.7 {\scriptsize$\pm$0.4} & 74.0 {\scriptsize$\pm$1.0} & 74.3 {\scriptsize$\pm$1.2} & 80.5 {\scriptsize$\pm$1.0} & \textbf{83.0} {\scriptsize$\pm$1.1} \\
Hidden only & 73.5 {\scriptsize$\pm$3.6} & 65.0 {\scriptsize$\pm$0.8} & \textbf{56.6} {\scriptsize$\pm$0.5} & \textbf{86.4} {\scriptsize$\pm$0.5} & 74.1 {\scriptsize$\pm$0.9} & 76.0 {\scriptsize$\pm$1.3} & \textbf{80.9} {\scriptsize$\pm$1.2} & 81.8 {\scriptsize$\pm$1.0} \\
\bottomrule
\end{tabular}
\caption{Complementary representation controls over ten paired seeds. Quality, relevance, and diversity use both views, the input view, or the hidden view; full-parameter utility and all other settings remain fixed.  Entries give per-task macro F1 mean $\pm$ sample standard deviation; bold indicates the best subset mean at the displayed precision.}
\label{tab:view-full}
\end{table*}

The lexical-only auxiliary uses input-space quality and diversity and first-layer gradient agreement. The hidden-only auxiliary uses hidden-space quality and diversity and output-layer agreement. Both keep full-parameter impact, uncertainty, stability, source queries, quotas, initialization, the utility budget, and the final-training rule identical to DVC in the primary comparison. Ten matched seeds therefore isolate which representation and parameter-block evidence enters the auxiliary correction.

\subsection{Tolerance over complete selection trajectories}
\label{app:tolerance}
\begin{figure*}[t]
\centering\includegraphics[width=\textwidth]{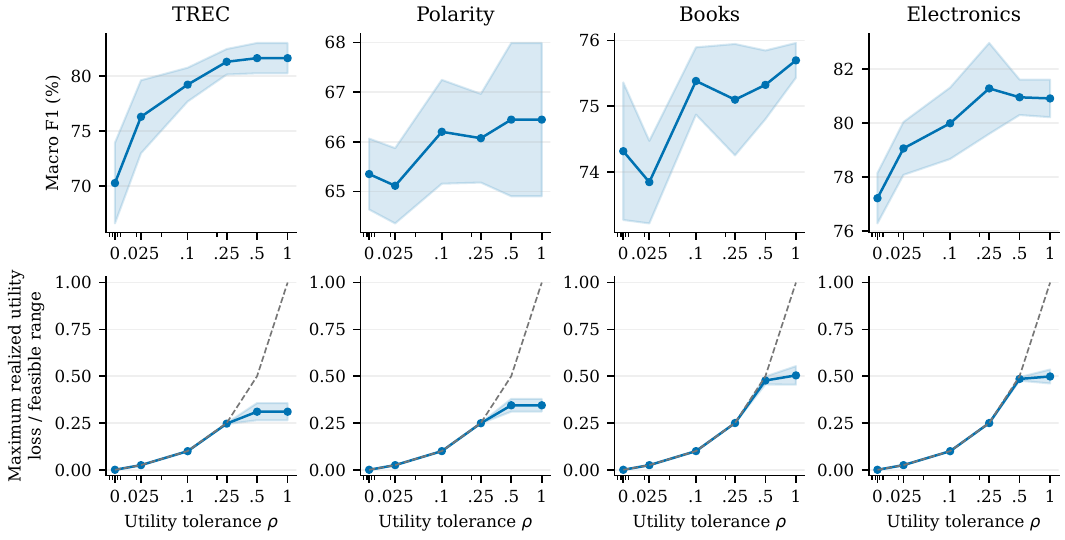}
\caption{Complete-trajectory tolerance sweep under a shared final-training protocol. Top: macro F1 after fresh training. Bottom: the largest realized utility loss divided by its feasible range across each run's selection rounds; the dashed curve is the allowance $\rho$. Points show means and sample standard deviations over five seeds. All other selection and training settings are fixed.}
\label{fig:tolerance}
\end{figure*}
\begin{table*}[t]
\centering
\footnotesize
\setlength{\tabcolsep}{2pt}
\begin{tabular}{lrrrrrrrr}
\toprule
$\rho$ & TREC F1 & Polarity F1 & Books F1 & Electronics F1 & TREC gap & Polarity gap & Books gap & Electronics gap \\
\midrule
0 & 70.3 {\scriptsize$\pm$3.7} & 65.4 {\scriptsize$\pm$0.7} & 74.3 {\scriptsize$\pm$1.1} & 77.2 {\scriptsize$\pm$0.9} & 0.0 {\scriptsize$\pm$0.0} & 0.0 {\scriptsize$\pm$0.0} & 0.0 {\scriptsize$\pm$0.0} & 0.0 {\scriptsize$\pm$0.0} \\
0.025 & 76.3 {\scriptsize$\pm$3.3} & 65.1 {\scriptsize$\pm$0.8} & 73.8 {\scriptsize$\pm$0.6} & 79.1 {\scriptsize$\pm$1.0} & 2.5 {\scriptsize$\pm$0.0} & 2.5 {\scriptsize$\pm$0.0} & 2.5 {\scriptsize$\pm$0.0} & 2.5 {\scriptsize$\pm$0.0} \\
0.1 & 79.2 {\scriptsize$\pm$1.5} & 66.2 {\scriptsize$\pm$1.1} & 75.4 {\scriptsize$\pm$0.5} & 80.0 {\scriptsize$\pm$1.3} & 10.0 {\scriptsize$\pm$0.0} & 10.0 {\scriptsize$\pm$0.0} & 10.0 {\scriptsize$\pm$0.0} & 10.0 {\scriptsize$\pm$0.0} \\
0.25 & 81.3 {\scriptsize$\pm$1.2} & 66.1 {\scriptsize$\pm$0.9} & 75.1 {\scriptsize$\pm$0.9} & 81.3 {\scriptsize$\pm$1.7} & 24.7 {\scriptsize$\pm$0.3} & 24.9 {\scriptsize$\pm$0.2} & 25.0 {\scriptsize$\pm$0.0} & 25.0 {\scriptsize$\pm$0.0} \\
0.5 & 81.6 {\scriptsize$\pm$1.4} & 66.4 {\scriptsize$\pm$1.5} & 75.3 {\scriptsize$\pm$0.5} & 81.0 {\scriptsize$\pm$0.7} & 31.0 {\scriptsize$\pm$4.6} & 34.5 {\scriptsize$\pm$3.4} & 47.7 {\scriptsize$\pm$2.1} & 48.5 {\scriptsize$\pm$1.0} \\
1 & 81.6 {\scriptsize$\pm$1.4} & 66.4 {\scriptsize$\pm$1.5} & 75.7 {\scriptsize$\pm$0.3} & 80.9 {\scriptsize$\pm$0.7} & 31.0 {\scriptsize$\pm$4.6} & 34.5 {\scriptsize$\pm$3.4} & 50.4 {\scriptsize$\pm$4.9} & 49.8 {\scriptsize$\pm$3.7} \\
\bottomrule
\end{tabular}
\caption{Complete-trajectory utility-tolerance analysis over five matched seeds. F1 is in percent. Gap is the maximum normalized utility loss across selection rounds within a run, in percent of the feasible utility range; entries give means and sample standard deviations across seeds. Only the tolerance changes. These are final-classifier results, rather than fixed-state batch measurements.}
\label{tab:tolerance-full}
\end{table*}

We prescribe $\rho\in\{0,0.025,0.1,0.25,0.5,1\}$ before evaluation and rerun complete trajectories on TREC, Polarity, Books, and Electronics. Unlike the fixed-state paths, changing a tolerance here also changes subsequent model states and source queries. Final F1 need not be monotone even though utility loss is monotone along a fixed-state scoring path. The sweep is explanatory and does not select the primary settings. At a fixed state, zero tolerance and removal of auxiliary scores both maximize reference utility. Equality of complete trajectories additionally requires the same tie resolution, example ordering, and random draws; the zero-tolerance point is therefore not an assertion of trajectory equivalence with a separate ablation.

\subsection{Component and structural analyses}
\label{app:components}
\begin{table*}[t]
\centering
\footnotesize
\setlength{\tabcolsep}{2pt}
\begin{tabular}{lrrrrrrrr}
\toprule
Variant & TREC & Polarity & 20 News & AG News & Books & DVD & Electronics & Kitchen \\
\midrule
DVC & 81.6 {\scriptsize$\pm$1.4} & 66.4 {\scriptsize$\pm$1.5} & 56.3 {\scriptsize$\pm$0.3} & 86.0 {\scriptsize$\pm$0.4} & 73.8 {\scriptsize$\pm$0.6} & 76.4 {\scriptsize$\pm$0.8} & 81.3 {\scriptsize$\pm$1.7} & 82.5 {\scriptsize$\pm$0.5} \\
No auxiliary terms & 68.0 {\scriptsize$\pm$2.5} & 64.9 {\scriptsize$\pm$1.0} & 55.9 {\scriptsize$\pm$0.7} & 86.2 {\scriptsize$\pm$0.3} & 73.9 {\scriptsize$\pm$1.5} & 73.3 {\scriptsize$\pm$1.5} & 78.2 {\scriptsize$\pm$2.2} & 81.6 {\scriptsize$\pm$1.0} \\
No quality gate & 74.0 {\scriptsize$\pm$7.0} & 65.8 {\scriptsize$\pm$1.1} & 56.3 {\scriptsize$\pm$0.6} & 86.6 {\scriptsize$\pm$0.6} & 73.1 {\scriptsize$\pm$0.8} & 74.9 {\scriptsize$\pm$0.9} & 79.8 {\scriptsize$\pm$0.9} & 83.0 {\scriptsize$\pm$0.6} \\
No relevance & 80.5 {\scriptsize$\pm$1.2} & 65.9 {\scriptsize$\pm$0.9} & 56.3 {\scriptsize$\pm$0.5} & 86.3 {\scriptsize$\pm$0.4} & 73.9 {\scriptsize$\pm$0.8} & 76.1 {\scriptsize$\pm$1.1} & 80.7 {\scriptsize$\pm$1.1} & 82.4 {\scriptsize$\pm$1.1} \\
No diversity & 80.3 {\scriptsize$\pm$0.9} & 66.0 {\scriptsize$\pm$0.8} & 56.3 {\scriptsize$\pm$0.4} & 85.9 {\scriptsize$\pm$0.5} & 73.8 {\scriptsize$\pm$1.1} & 75.4 {\scriptsize$\pm$0.5} & 80.7 {\scriptsize$\pm$1.0} & 82.3 {\scriptsize$\pm$0.9} \\
No uncertainty & 75.8 {\scriptsize$\pm$2.0} & 65.9 {\scriptsize$\pm$1.3} & 56.4 {\scriptsize$\pm$0.7} & 85.8 {\scriptsize$\pm$0.6} & 73.7 {\scriptsize$\pm$1.4} & 75.7 {\scriptsize$\pm$0.9} & 80.7 {\scriptsize$\pm$0.6} & 82.0 {\scriptsize$\pm$1.1} \\
No stability & 80.5 {\scriptsize$\pm$1.2} & 66.6 {\scriptsize$\pm$0.4} & 56.4 {\scriptsize$\pm$0.6} & 86.2 {\scriptsize$\pm$0.3} & 75.0 {\scriptsize$\pm$1.0} & 74.5 {\scriptsize$\pm$1.0} & 80.8 {\scriptsize$\pm$1.6} & 82.3 {\scriptsize$\pm$0.4} \\
No impact & 80.2 {\scriptsize$\pm$1.2} & 66.3 {\scriptsize$\pm$0.9} & 56.3 {\scriptsize$\pm$0.4} & 86.2 {\scriptsize$\pm$0.2} & 74.8 {\scriptsize$\pm$1.0} & 75.7 {\scriptsize$\pm$1.1} & 80.8 {\scriptsize$\pm$0.6} & 82.1 {\scriptsize$\pm$1.1} \\
Uniform queries & 80.0 {\scriptsize$\pm$1.8} & 66.3 {\scriptsize$\pm$0.9} & 56.1 {\scriptsize$\pm$0.8} & 85.0 {\scriptsize$\pm$0.3} & 74.2 {\scriptsize$\pm$1.6} & 75.6 {\scriptsize$\pm$0.8} & 80.9 {\scriptsize$\pm$0.6} & 82.1 {\scriptsize$\pm$1.4} \\
Compact auxiliary & 80.6 {\scriptsize$\pm$1.2} & 66.0 {\scriptsize$\pm$1.1} & 56.4 {\scriptsize$\pm$0.5} & 85.9 {\scriptsize$\pm$0.2} & 74.3 {\scriptsize$\pm$1.6} & 75.7 {\scriptsize$\pm$1.2} & 80.3 {\scriptsize$\pm$0.7} & 81.3 {\scriptsize$\pm$1.2} \\
\bottomrule
\end{tabular}
\caption{Complete component comparison: macro F1 mean $\pm$ sample standard deviation over five paired seeds. Each control retains DVC initialization, quotas, coefficients, tolerances, and final-training protocol. The no-impact control removes the ranking utility and its calibration. Single-signal deletions average the remaining auxiliary signals; compact drops uncertainty and stability jointly.}
\label{tab:component-full}
\end{table*}

The component controls use the primary comparison's shared final-training protocol and its first five matched seeds. Each deletes the named signal while preserving all other choices; the compact variant drops uncertainty and stability together. The impact deletion removes its within-batch ranking term and disables the utility budget; the source-query feedback remains gradient based. Uniform queries keep all scoring and coverage rules. These controls isolate contributions within the full hierarchy, rather than retuning reduced models.

\begin{table*}[t]
\centering
\footnotesize
\setlength{\tabcolsep}{2pt}
\begin{tabular}{lrrrrrrrr}
\toprule
Variant & TREC & Polarity & 20 News & AG News & Books & DVD & Electronics & Kitchen \\
\midrule
DVC & 79.1 {\scriptsize$\pm$3.7} & 66.3 {\scriptsize$\pm$1.1} & 56.4 {\scriptsize$\pm$0.4} & 86.2 {\scriptsize$\pm$0.1} & 75.3 {\scriptsize$\pm$0.9} & 76.2 {\scriptsize$\pm$0.7} & 81.7 {\scriptsize$\pm$1.5} & 82.5 {\scriptsize$\pm$1.3} \\
No utility constraint & 79.1 {\scriptsize$\pm$3.7} & 66.3 {\scriptsize$\pm$1.1} & 56.4 {\scriptsize$\pm$0.4} & 86.3 {\scriptsize$\pm$0.3} & 74.3 {\scriptsize$\pm$1.2} & 76.2 {\scriptsize$\pm$0.7} & 81.0 {\scriptsize$\pm$0.8} & 82.0 {\scriptsize$\pm$0.8} \\
Keep initial subset & 71.1 {\scriptsize$\pm$4.4} & -- & -- & -- & -- & -- & -- & -- \\
No class quotas & 70.6 {\scriptsize$\pm$7.8} & 66.4 {\scriptsize$\pm$1.1} & 56.5 {\scriptsize$\pm$0.6} & 86.8 {\scriptsize$\pm$0.2} & 74.5 {\scriptsize$\pm$1.7} & 76.9 {\scriptsize$\pm$0.4} & 81.1 {\scriptsize$\pm$1.0} & 82.6 {\scriptsize$\pm$0.8} \\
No prior calibration & 70.5 {\scriptsize$\pm$8.3} & 66.3 {\scriptsize$\pm$1.1} & 56.2 {\scriptsize$\pm$0.4} & 86.0 {\scriptsize$\pm$0.2} & 74.7 {\scriptsize$\pm$1.2} & 76.2 {\scriptsize$\pm$1.3} & 81.1 {\scriptsize$\pm$0.7} & 82.3 {\scriptsize$\pm$0.7} \\
\bottomrule
\end{tabular}
\caption{Structural controls in the supplementary method-specific comparison: macro F1 mean and sample standard deviation over 5 seeds per setting. Each control keeps its matched DVC settings except the named change; removing the utility constraint retains the same maximum correction. Initial-subset retention is evaluated on TREC only; dashes denote unreported combinations.}
\label{tab:structural-controls}
\end{table*}

Table~\ref{tab:structural-controls} varies class quotas, prior adjustment, initial-set retention, and the utility constraint under the supplementary method-specific settings, holding other settings fixed across five seeds. On TREC, removing quotas or prior calibration lowers F1 from 79.1 to 70.6 or 70.5. Retaining all 256 initialization examples gives 71.1, versus 79.1 when they can be reselected.

\begin{table}[tbp]
\centering
\small
\setlength{\tabcolsep}{4pt}
\begin{tabular}{lrrr}
\toprule
Selector & Observed & Uncorrupted & Noise (\%) \\
\midrule
DVC & 14.0 & 11.8 & 2.0 \\
Random & 53.2 & 10.5 & 22.6 \\
GLISTER-style & 21.4 & 9.3 & 5.0 \\
Flat six-signal & 4.3 & 4.2 & 1.8 \\
Alignment & 15.3 & 2.9 & 11.2 \\
\bottomrule
\end{tabular}
\caption{TREC abbreviation-class support in the primary comparison at 20\% retention, averaged over ten seeds. Observed counts examples labeled ABBR; uncorrupted counts retained ABBR examples whose observed label equals the corpus label. Noise is the corruption percentage of the entire retained subset.}
\label{tab:class-support}
\end{table}

Symmetric corruption inflates the rare abbreviation category's observed count (Table~\ref{tab:class-support}). Compared with random retention, \DVC\ keeps fewer observed abbreviation labels but more uncorrupted questions. Its better clean-class coverage coexists with similar overall purity to flat scoring, revealing useful supervision hidden by aggregate purity in these constructed pools.

The maximum coefficient sets the available correction strength; calibration limits its utility cost. At equal maximum coefficients, calibration preserves comparable F1 while bounding the local utility exchange. Development-selected fixed coefficients give 75.0 eight-task F1 versus calibrated \DVC's 75.2: a ten-seed difference of 0.2 ($[-0.1,0.5]$). This comparison tests the utility constraint at competitive final-subset performance.

\subsection{Reference-size sensitivity}
\label{app:controlled-reference}
On TREC and Polarity, we reserve 512 stratified target examples and draw three unstratified 128-example subsets. Unused references stay withheld. Source, development, test, and feature partitions remain identical across reference sizes and clean/mixed source conditions, isolating reference information from source-pool changes.

At 128 references, fixed broad-comparison DVC settings are compared with independently adapted DVC, independently adapted calibrated quality, and random selection. Each adapted selector receives 18 distinct configurations from the coefficient, tolerance, neighborhood, reference-pretraining, and initialization-retention choices in Appendix~\ref{app:supplement}. Mean development F1 on three paired training-seed/reference-composition runs selects one setting per task and source condition. The reference objective remains unweighted. At 512 references, DVC and calibrated quality use the broad-comparison DVC settings; neither is independently adapted in this diagnostic.

Evaluation crosses five training seeds with the three 128-example references; compositions are averaged within each seed before computing variability and paired intervals. The 512-example condition uses five seeds. Development and evaluation share reference compositions, so the analysis concerns these fixed draws. All TREC classes occur in each draw: ABBR has three, two, and one examples, versus eight at 512. Polarity counts are $(67,61)$, $(68,60)$, and $(58,70)$, versus $(256,256)$.

Larger references improve fixed DVC in all four conditions (Table~\ref{tab:reference-nested}). Paired gains are 14.5 F1 points on both clean and mixed TREC, and 1.8 and 4.2 on Polarity; all four 95\% intervals exclude zero. Independent adaptation improves mixed TREC at 128 references by 8.2 points, but the adapted methods remain below random selection in these conditions. An empirical utility constraint cannot guarantee reliable guidance from a small target sample.

On mixed TREC, fixed DVC retains 5.2 clean ABBR examples with 128 references and 11.0 with 512; ABBR F1 rises from 10.9 to 43.2. On clean sources, retained ABBR counts are similar (14.3 versus 14.0), yet F1 rises from 41.9 to 69.3. These observations show that reference information matters beyond class presence and retained class counts alone.

\begin{table*}[t]
\centering\small
\setlength{\tabcolsep}{5pt}
\begin{tabular}{llrrrrr}
\toprule
Task & Sources & Reference & Fixed DVC & Tuned DVC & Calibrated quality & Random \\
\midrule
TREC & clean & 128 & 62.4 $\pm$ 5.5 & 63.7 $\pm$ 3.1 & 72.3 $\pm$ 1.0 & 80.0 $\pm$ 2.7 \\
TREC & clean & 512 & 76.9 $\pm$ 7.8 & -- & 68.9 $\pm$ 9.2 & 80.0 $\pm$ 2.7 \\
TREC & mixed & 128 & 56.5 $\pm$ 3.0 & 64.7 $\pm$ 4.3 & 61.9 $\pm$ 4.3 & 70.2 $\pm$ 2.9 \\
TREC & mixed & 512 & 71.0 $\pm$ 7.6 & -- & 62.9 $\pm$ 7.2 & 70.2 $\pm$ 2.9 \\
Polarity & clean & 128 & 66.8 $\pm$ 0.3 & 66.3 $\pm$ 0.6 & 65.4 $\pm$ 0.6 & 69.5 $\pm$ 0.7 \\
Polarity & clean & 512 & 68.6 $\pm$ 0.6 & -- & 68.2 $\pm$ 0.3 & 69.5 $\pm$ 0.7 \\
Polarity & mixed & 128 & 62.1 $\pm$ 0.6 & 61.5 $\pm$ 1.1 & 61.3 $\pm$ 0.5 & 62.9 $\pm$ 1.3 \\
Polarity & mixed & 512 & 66.2 $\pm$ 0.5 & -- & 65.5 $\pm$ 0.6 & 62.9 $\pm$ 1.3 \\
\bottomrule
\end{tabular}
\caption{Controlled reference-size diagnosis at 20\% source retention. Source, development, test, and feature partitions are fixed. At 128 references, means and standard deviations use five training-seed means, each averaged over three reference compositions. At 512 references there is one shared reference set and five training seeds. Calibrated quality is independently tuned at 128 and inherits fixed DVC settings at 512; tuned DVC is evaluated only at 128.}
\label{tab:reference-nested}
\end{table*}

\subsection{Held-out-loss and simple-score comparisons}
\label{app:supplement}

\paragraph{Reducible held-out loss.}
RHO-LOSS subtracts a frozen reference-model loss from current training loss \citep{pmlr-v162-mindermann22a}. An independently initialized same-width MLP fits reference labels for 400 AdamW updates (learning rate $0.003$, weight decay $10^{-4}$). Each round uniformly queries at most 256 candidates, retains 32, and updates the selector on accumulated examples. Reference-model fitting is included in selection time.

\paragraph{Quality corrections.}
Alignment plus quality uses $I_i+\lambda\bar Q_i$ with DVC's impact normalization, neighborhood support, source queries, quotas, and initialization. Calibrated quality applies the same utility-budget search with $H_i=\bar Q_i$. In the primary comparison, the latter inherits DVC's other selection settings; the former searches its coefficient and inherits the remaining selection settings. These controls isolate the scoring rule at the inherited settings, rather than establish the best independently optimized simple selector.

\paragraph{Development of the broad-comparison controls.}
Seeds 101, 103, and 107 select configurations by mean development macro F1, without test access. Alignment plus quality searches $\lambda\in\{0.25,0.5,1,2,4\}$ and retains DVC's final-training rule. RHO-LOSS searches 0/20/50 reference-initialization updates and initial-set retention, keeping 400 loss-model updates. It then searches final learning rates $\{0.001,0.003\}$ and class-weighting powers $\{0,0.5,1\}$ (zero on balanced Amazon tasks). Configurations are fixed before the separate six-method, ten-seed evaluation. The primary shared-training comparison uses the same selection choices with Table~\ref{tab:fair-recipes}'s final-training rule.

\begin{table}[t]
\centering
\footnotesize
\setlength{\tabcolsep}{2pt}
\begin{tabular}{lrrr}
\toprule
Task & $\lambda_Q$ & RHO init. & Initial set \\
\midrule
TREC & 2 & 0 & keep \\
Polarity & 1 & 0 & keep \\
20 News & 2 & 0 & release \\
AG News & 0.25 & 50 & keep \\
Books & 1 & 0 & keep \\
DVD & 1 & 0 & keep \\
Electronics & 1 & 0 & keep \\
Kitchen & 0.5 & 0 & keep \\
\bottomrule
\end{tabular}
\caption{Comparator selection settings. Alignment plus quality inherits all other DVC selection settings. RHO-LOSS uses a frozen loss model fitted for 400 updates; final training follows Table~\ref{tab:fair-recipes}.}
\label{tab:scoring-settings}
\end{table}

\paragraph{Independent tuning of structure and calibration.}
A separate comparison crosses full auxiliary evidence or quality-only evidence with a fixed coefficient or local calibration. Each method receives 36 distinct configurations per task, each evaluated on the same three development seeds. Candidate choices are $\lambda_{\max}\in\{0.25,0.5,1,2,4\}$, $\rho\in\{0.025,0.1,0.25,0.5,1\}$ when calibrated, $k\in\{4,8,16,32\}$, reference initialization in $\{0,20,50\}$ updates, and initial-set retention/release; quotas remain enabled. Fixed-coefficient methods use $\rho=1$. The budget is the number of evaluated configurations, not an exhaustive search over the candidate product. Development mean F1 determines the selected setting, and final training retains the shared task settings.

Within this comparison, accumulated and final training examples follow their original sample-ID order before shared seeded minibatch draws. Thus score ordering cannot change training order for identical memberships. Diversity memory still follows actual selection order. The ten evaluation seeds form a separate paired block; they are not combined with numerical records from the shared-setting comparison.

\subsection{Global lexical pruning with shared final training}
\label{app:geometry}

Five matched seeds compare SCDP \citep{nguyen-he-2025-swift}, random, GLISTER-style, flat scoring, and DVC at 20\% retention, sharing each task's full-pool development settings. Reference examples guide target-dependent selectors but are excluded from final training.

SCDP uses the released default unigram TF--IDF and global geometric median: it retains farthest examples for budgets up to 1,500 and otherwise samples across 100 distance bins, allocating first to smaller bins. Median tolerance is $10^{-5}$. Amazon uses provided lexical counts. This geometric control accesses the full pool; adaptive selectors use the 256-candidate round cap.

Eight-task F1 is 75.6 for DVC, 74.5 for GLISTER, 74.2 for flat scoring, 71.5 for random, and 67.5 for SCDP. Paired DVC gains are 1.1 over GLISTER ($[0.4,1.8]$) and 1.3 over flat scoring ($[0.4,2.3]$). This comparison evaluates target guidance beyond global lexical spread.

\subsection{Hierarchy versus calibrated quality correction}
\label{app:hierarchy}
\begin{table}[tbp]
\centering\small
\setlength{\tabcolsep}{4pt}
\begin{tabular}{lrrr}
\toprule
$\rho$ & $\Delta$ purity & $\Delta$ coverage & $\Delta$ loss gain \\
\midrule
0 & 0.0 & 0.0 & 29.0 \\
0.025 & +0.8 & 0.4 & -15.0 \\
0.1 & +0.6 & -1.5 & -292.0 \\
0.25 & -0.5 & -4.0 & -175.0 \\
0.5 & -0.3 & -6.4 & -274.0 \\
1 & -0.3 & -6.4 & -274.0 \\
Primary & -0.3 & -4.4 & -249.0 \\
\bottomrule
\end{tabular}
\caption{Full hierarchy minus quality correction at a matched utility loss. Purity is in percentage points over TREC and Polarity; coverage and online loss gains average all four tasks. Coverage and loss gains are multiplied by $10^3$ and $10^6$, respectively. Each mean averages three states within five seeds. Primary uses each task's selected tolerance.}
\label{tab:hierarchy-local}
\end{table}

Three fixed states on TREC, Polarity, Books, and Electronics, five seeds, and six tolerances yield 1,080 batches across three corrections. Quality correction receives either the same allowance or the hierarchy's realized gap; the latter preserves at least as much reference utility (Table~\ref{tab:hierarchy-local}). Clean coverage averages each development text's largest nonnegative lexical cosine to a selected example with the same observed label and an uncorrupted corpus label; missing matches score zero. Labels are inspected only after selection. Online loss gains use five SGD updates on the accumulated subset plus the batch, preserving learning rate, sampling seed, precision, and order.

At primary tolerances, the hierarchy's loss gain is smaller by $0.00025$ ($[-0.00048,-0.00002]$), and coverage by $0.0044$ ($[-0.0071,-0.0017]$). Score order and position-based sampling can slightly separate identical memberships. On TREC, quality correction retains 13.0 clean abbreviation questions and 1.1\% corruption, versus 11.6 and 1.9\% for the hierarchy. Local purity, coverage, and short-update loss therefore measure different aspects of selection from final-classifier F1; the independently tuned comparison evaluates the latter directly.

\subsection{Convolutional transfer and computational cost}
\label{app:primary-cnn}
\begin{table*}[tbp]
\centering\small
\setlength{\tabcolsep}{5pt}
\begin{tabular}{lrrrrr}
\toprule
Selector & TREC & Polarity & 20 News & AG News & Mean \\
\midrule
\multicolumn{6}{l}{Macro F1} \\
Random & $62.4\pm6.0$ & $58.0\pm1.2$ & $24.1\pm0.8$ & $80.1\pm0.8$ & $56.2\pm1.7$ \\
GLISTER-style & $73.9\pm3.4$ & $61.0\pm0.8$ & $37.1\pm1.0$ & $83.1\pm0.4$ & $63.8\pm0.8$ \\
Flat six-signal & $65.5\pm2.6$ & $62.3\pm1.2$ & $38.2\pm0.5$ & $85.5\pm0.4$ & $62.9\pm0.9$ \\
Calibrated quality & $75.8\pm3.4$ & $62.8\pm0.7$ & $37.5\pm0.6$ & $81.2\pm0.9$ & $64.3\pm0.8$ \\
DVC & $76.3\pm3.2$ & $62.4\pm1.1$ & $37.8\pm0.9$ & $81.5\pm1.2$ & $64.5\pm0.9$ \\
\bottomrule
\end{tabular}
\caption{Fresh CNN training on subsets selected in the primary comparison. Macro F1 entries are mean $\pm$ sample standard deviation over ten matched seeds, in percent. All five selectors retain 20\% of source examples and share the same CNN settings within each task. The final column averages tasks within each seed. The CNN learns embeddings from scratch; reference examples are excluded from its training.}
\label{tab:primary-cnn}
\end{table*}

A fresh CNN trains on the primary MLP-selected texts, with shared settings, no reference examples, and ten matched seeds. Four-task F1 is 64.5 for DVC, 64.3 for calibrated quality, 63.8 for GLISTER, 62.9 for flat scoring, and 56.2 for random. Paired DVC gains are 0.2 ($[-0.8,1.2]$), 0.7 ($[-0.1,1.5]$), 1.6 ($[0.7,2.6]$), and 8.4 ($[7.2,9.5]$), respectively. The first two intervals permit either ordering; the subsets preserve an advantage over random and flat scoring. Selection settings are fixed before CNN evaluation. 

\begin{table}[t]
\centering
\footnotesize
\setlength{\tabcolsep}{3pt}
\begin{tabular}{lrrrr}
\toprule
Task & Method & Select & Train & Scored \\
\midrule
TREC & Random & 0.0 & 64.0 & 0 \\
TREC & Alignment & 46.0 & 62.0 & 5,376 \\
TREC & GLISTER-style & 265.0 & 63.0 & 7,424 \\
TREC & DVC & 179.0 & 63.0 & 7,424 \\
Polarity & Random & 0.0 & 61.0 & 0 \\
Polarity & Alignment & 92.0 & 62.0 & 10,496 \\
Polarity & GLISTER-style & 370.0 & 61.0 & 10,496 \\
Polarity & DVC & 315.0 & 61.0 & 10,496 \\
20 News & Random & 0.0 & 123.0 & 0 \\
20 News & Alignment & 118.0 & 123.0 & 14,336 \\
20 News & GLISTER-style & 533.0 & 123.0 & 14,336 \\
20 News & DVC & 475.0 & 124.0 & 14,336 \\
\bottomrule
\end{tabular}
\caption{Median CPU timings at 20\% budget (3 seeds per setting), measured sequentially with one PyTorch thread on an Apple M5. Selection includes the initialization, reference pretraining, source queries, scoring, and online updates required by each method. Scored counts candidate evaluations, including repeat queries. Select and Train are milliseconds. Loading, lexical preprocessing, and final evaluation are excluded.}
\label{tab:timing}
\end{table}

Table~\ref{tab:timing} separates selection cost from final-classifier training. These compact MLPs are inexpensive to train, so selecting a reusable subset is an additional cost rather than an end-to-end speedup in this setting. Lexical feature fitting scans all training texts once. Sequential timings use Apple M5; larger-pool runs use Intel Xeon Platinum 8362. CNNs use NVIDIA A100; the timing comparison concerns sequential CPU runs. Exact contractions avoid a candidate-by-parameter gradient matrix, while feature storage and reference comparisons remain material costs.

\subsection{Fixed-state utility and reliability}
\label{app:tradeoff}
Five seeds follow the primary settings. At the first, middle, and final states, candidates, memory, snapshots, and quotas are frozen. Uncalibrated paths vary $\lambda\in\{0,0.25,0.5,1,2,4\}$; tolerance paths retain the primary maximum coefficient and vary $\rho\in\{0,0.025,0.1,0.25,0.5,1\}$. Counterfactual batches do not change later updates, and development/test labels never choose path points.

Figure~\ref{fig:tradeoff} summarizes quota-constrained utility loss and label support. Purity uses synthetic corruption only after selection. Amazon support uses reference neighborhoods and is not independent label-accuracy evidence. Standard deviations are computed across seeds after averaging their three states. Every batch satisfies its cap and Equation~\ref{eq:exchange}; auxiliary improvement need not mean more reliable labels.

A full-batch double-precision update tests the utility approximation. At step $10^{-3}$, predicted and actual reference-loss decreases have pooled correlation 0.99992, within-state correlations 0.98032--1.00000, and matching signs. Reducing $10^{-2}$ to $10^{-3}$ reduces the absolute Taylor remainder by a median factor 100.0, consistent with Equation~\ref{eq:local-update}. At $10^{-2}$, pooled correlation is 0.98919; curvature and gradient norms can change within-state rankings. 

\subsection{Fixed-state text cases}
\label{app:cases}
Development seed 101 freezes TREC rounds 1, 15, and 29 under the primary settings. Candidate scores, memories, snapshots, and quotas are held fixed. The illustrated pairs maximize the quality difference among same-observed-class exchanges that remove a synthetically corrupted example and add an uncorrupted one; the initial state has no such pair. Corpus labels reveal introduced corruption without certifying the original annotations.

The batches contain 32, 32, and 24 examples, with 7, 17, and 18 exchanges. Corrupted examples removed/added are 0/0, 2/1, and 8/0. Gradient-only purity is 100.0\%, 93.8\%, and 62.5\%, versus DVC's 100.0\%, 96.9\%, and 95.8\%. Utility losses are 0.5\%, 9.6\%, and 21.0\% of the feasible range. These counts cover all saved states, including the unchanged early-state purity.

\FloatBarrier

\section{Further Related Work}
\label{app:related}
\paragraph{Text relevance and annotation quality.}
Corpus matching, dynamic selection, and learned similarity--diversity scores adapt training text to a target distribution \citep{moore-lewis-2010-intelligent,axelrod-etal-2011-domain,van-der-wees-etal-2017-dynamic,ruder-plank-2017-learning}. Noise-aware selection addresses weak supervision \citep{ren-etal-2020-denoising,lang2022weaksubset}, augmented pseudo-labels \citep{liu-etal-2022-dynamic-augmentation}, and textual label noise \citep{qiao-etal-2022-selfmix}. Recent methods combine diversity with quality assessment or undersample imbalanced classes \citep{fonseca-etal-2025-instance,zeng-etal-2025-data}. DVC uses prior-adjusted label support to gate auxiliary evidence without changing labels.

\paragraph{Geometry and learning dynamics.}
Geometric coresets and gradient diversity preserve coverage \citep{sener2018coreset,ash2020badge}; training dynamics, early errors, and reducible loss measure usefulness \citep{swayamdipta-etal-2020-dataset,paul2021diet,pmlr-v162-mindermann22a}. SCDP prunes through global lexical geometry \citep{nguyen-he-2025-swift}. Large-scale studies also find random selection competitive \citep{xia-etal-2025-rethinking}. DVC grounds these complementary signals in reference-gradient utility and bounds the cost of their correction.

\paragraph{Data valuation and adaptive mixtures.}
Influence functions, coalition values, gradient trajectories, and reinforcement learning quantify training contributions \citep{pmlr-v70-koh17a,pmlr-v97-ghorbani19c,pruthi2020tracin,pmlr-v119-yoon20a}. Resampling and reweighting adapt source mixtures \citep{xie2023dsir,xie2023doremi,pmlr-v235-fan24e}. LESS and LearnAlign use gradient alignment for targeted instruction tuning and reinforcement learning, respectively \citep{pmlr-v235-xia24c,li-etal-2026-learnalign}. TAGCOS uses task-agnostic gradient clustering \citep{zhang-etal-2025-tagcos}, ClusterUCB allocates selection effort \citep{wang-etal-2025-clusterucb}, and GAST selects data and layers \citep{yao-etal-2026-gast}. DVC couples source queries and hierarchical evidence to select reusable texts under a local utility constraint.

\end{document}